\documentclass[sigconf]{acmart}

\usepackage{balance}

\def\BibTeX{{\rm B\kern-.05em{\sc i\kern-.025em b}\kern-.08emT\kern-.1667em\lower.7ex\hbox{E}\kern-.125emX}}

\usepackage{booktabs} 
\usepackage{subcaption}
\usepackage{algorithm}
\usepackage{tikz}
\usepackage{algpseudocode}
\usepackage{amssymb}
\usepackage{amsmath}
\usepackage{amsthm}
\usepackage{algorithmicx}
\usepackage{verbatim}
\usepackage{bm}
\usepackage{dsfont}
\usepackage{url}
\usepackage{mathtools}

\usepackage[T1]{fontenc}
\usepackage{array}
\usepackage{multirow}

\usepackage{aecompl}

\newcommand{\mt}{\mathsf{T}}

\newcommand{\compilehidecomments}{false}

\ifthenelse{ \equal{\compilehidecomments}{true} }{%
	\newcommand{\wei}[1]{}
}{
\newcommand{\wei}[1]{{\color{blue!50!black}  [\text{Wei:} #1]}}
}

\ifthenelse{ \equal{\compilehidecomments}{true} }{%
	\newcommand{\qingyun}[1]{}
}{
\newcommand{\qingyun}[1]{{\color{green!50!black}  [\text{Qingyun:} #1]}}
}

\ifthenelse{ \equal{\compilehidecomments}{true} }{%
	\newcommand{\summary}[1]{}
}{
\newcommand{\summary}[1]{{\color{red!60!black}  [\text{Summary:} #1]}}
}

\ifthenelse{ \equal{\compilehidecomments}{true} }{%
	\newcommand{\huazheng}[1]{}
}{
\newcommand{\huazheng}[1]{{\color{blue!50!black}  [\text{Huazheng:} #1]}}
}

\ifthenelse{ \equal{\compilehidecomments}{true} }{%
	\newcommand{\revision}[1]{}
}{
\newcommand{\revision}[1]{{\color{red!50!black}  \text{} #1}}
}

\DeclareMathOperator*{\argmax}{arg\,max}

\def \btheta {\bm \theta}

\def \btheta {\mathrm{\boldsymbol{\theta}}}

\def \bA {\mathbf{A}}

\def \bb {\mathbf{b}}
\def \bC {\mathbf{C}}
\def \cC {\mathcal{C}}
\def \bd {\mathbf{d}}

\def \bP {\mathbf{P}}
\def \bI {\mathbf{I}}
\def \bS {\mathbf{S}}

\def \cS {\mathcal{S}}
\def \cE {\mathcal{E}}

\def \cV {\mathcal{V}}

\def \Beta {\bm{\beta}}

\def \bbE {\mathbb{E}}
\def \bbP {\mathbb{P}}

\def \giving {\text{giving}}

\def \receiving {\text{receiving}}

\def \textin {\text{in}}

\def \textout {\text{out}}

\setcopyright{none}
\begin{document}

\copyrightyear{2019} 
\acmYear{2019} 
\setcopyright{acmcopyright}
\acmConference[KDD '19]{The 25th ACM SIGKDD Conference on Knowledge Discovery and Data Mining}{August 4--8, 2019}{Anchorage, AK, USA}
\acmPrice{15.00}
\acmDOI{10.1145/3292500.3330874}
\acmISBN{978-1-4503-6201-6/19/08}

\title{Factorization Bandits for Online Influence Maximization}

\author{Qingyun Wu}
\affiliation{%
  \institution{ University of Virginia }
  \streetaddress{85 Engineer's Way}
  \city{Charlottesville} 
  \state{VA, USA} 
  \postcode{22904}
}
\email{qw2ky@virginia.edu}

\author{Zhige Li}
\affiliation{%
  \institution{ Shanghai Jiao Tong University}
  \city{Shanghai} 
  \state{China} 
}
\email{l-zhige@outlook.com}

\author{Huazheng Wang}
\affiliation{%
  \institution{ University of Virginia }
  \streetaddress{85 Engineer's Way}
  \city{Charlottesville} 
  \state{VA, USA} 
  \postcode{22904}
}
\email{hw7ww@virginia.edu}

\author{Wei Chen}
\affiliation{%
  \institution{ Microsoft Research }
  \city{Beijing}
  \state{China} 
}
\email{weic@microsoft.com}

\author{Hongning Wang}
\affiliation{%
  \institution{ University of Virginia }
  \city{Charlottesville} 
  \state{VA, USA} 
}
\email{hw5x@virginia.edu}	

\begin{abstract}
    In this paper, we study the problem of \emph{online influence maximization} in social networks. In this problem, a learner aims to identify the set of ``best influencers'' in a network by interacting with the network, i.e., repeatedly selecting seed nodes and observing activation feedback in the network. We capitalize on an important property of the influence maximization problem named \textbf{\emph{network assortativity}}, which is ignored by most existing works in online influence maximization. To realize network assortativity, we factorize the activation probability on the edges into latent factors on the corresponding nodes, including influence factor on the giving nodes and susceptibility factor on the receiving nodes. We propose an upper confidence bound based online learning solution to estimate the latent factors, and therefore the activation probabilities. Considerable regret reduction is achieved by our factorization based online influence maximization algorithm. Extensive empirical evaluations on two real-world networks showed the effectiveness of our proposed solution. 
\end{abstract}

%
%

\ccsdesc[500]{Information systems~Social advertising}
\ccsdesc[500]{Theory of computation~Online learning algorithms}
\ccsdesc[500]{Theory of computation~Social networks}

\keywords{Online influence maximization; Factorization bandit; Network assortativity; Regret analysis}

\maketitle

\section{Introduction}

Online social networks play a vital role in the spread of information, ideas, and influence among people in their modern life \cite{mislove2007measurement,mcpherson2001birds}. They have been actively utilized as a dissemination and marketing platform. For instance, in viral marketing, a marketer tries to select a set of customers with great influence for a new product promotion. 
With a fixed budget on the number of selections, a marketer aims to maximize the spread of this influence, which is referred to as the influence maximization problem \cite{kempe2003maximizing,kitsak2010identification,centola2007complex}. In typical solutions for influence maximization, a social network is modeled as a graph with nodes representing users and edges, associated with the activation probability, representing the connections or relationship between users. Influence is propagated through the  network under a specific diffusion model, such as independent cascade model and linear threshold model \cite{kempe2003maximizing}. 

Most existing influence maximization solutions assume the activation probability is known beforehand. However, in many real-world social networks, this information is not observable. Solutions have been proposed to estimate the activation probabilities from a set of cascades which consist of logged actions by the network users in the past \cite{goyal2010learning, Netrapalli:2012:LGE:2318857.2254783,Saito:2008:PID:1430307.1430318, bourigault2016representation}. However, as the observations are independently collected from the learning algorithms, bias might be introduced by the logging mechanism \cite{morgan2014counterfactuals,el2012social}. 
To combat with biases in offline influence estimation, Aral and Walker \cite{Aral2018} proposed randomized online experiments for offline data collection. But it is usually very expensive to carry out such experiments. 
This motivates the studies of online influence maximization \cite{chen2013combinatorial, Chen:2016:CMB:2946645.2946695, vaswani2015influence, wen2017online, vaswani2017model}, in which seed nodes are purposely selected by a learning agent to improve its quality in influence estimation  and influencer selection on the fly. The foundation of this line of solutions is the combinatorial bandits \cite{chen2013combinatorial}, in which a set of arms are pulled together at each round, and the outcome is only revealed as a whole over the set of pulled arms. Mapping it back to the influence maximization problem, each node in the network is considered as an arm, and at each round the received reward on the selected set of seed nodes is the number of their activated nodes.    


Most of existing online influence maximization solutions \cite{chen2013combinatorial, Chen:2016:CMB:2946645.2946695} model the activation probability on the edges independently, which unfortunately cannot capture how social influence forms 
in real networks \cite{newman2002assortative,Aral337}. 
To clarify, here we are referring to the independence/dependence in the estimation of influence across network edges instead of how influence diffuses.  
First, this independence assumption prevents the model from realizing the dependency among the influence pattern on related nodes, i.e., the so-called assortative mixing \cite{newman2002assortative}.
For example, in ``influentials hypothesis'' \cite{watts2007influentials} the tendency that a user is likely to influence his/her neighbors is defined as the \emph{influence} of this node and the tendency that he/she is likely to be influenced is defined as the node's \emph{susceptibility}. In Aral and Walker's work \cite{Aral337}, via randomized experiments, they found that a node's influence should be separated from its susceptibility: influential individuals are less susceptible to be influenced than non-influential individuals. This directly suggests the activation probability on an edge should be modeled by both its end nodes' influence and susceptibility, and thus edges that share the same set of giving or receiving nodes are no longer independent from each other. 
Second, the distribution of influence and susceptibility over network nodes is heterogeneous. It is also reported in \cite{Aral337} that influential individuals cluster in the network while susceptible individuals do not. Hence, a joint modeling of nodes' influence and susceptibility is necessary. In a recent work, Aral and Dhillon \cite{Aral2018} showed that failing to differentiate influence and susceptibility across network nodes caused an influence maximization algorithm to underestimate influence propagation by $21.7\%$ on average, for a fixed seed set size. To overcome the issue, some online influence maximization algorithms \cite{wen2017online, vaswani2017model, pmlr-v80-kalimeris18a} introduced edge-level features to help the modeling of influence probability. However, as revealed in \cite{Aral337}, the factors affecting users' influence and susceptibility may include age, gender, marital status and many other sensitive attributes, which can hardly be exhausted and are often prohibited under privacy constraints. Third, edge-level estimation of activation probability costs an algorithm both high computational complexity and sample complexity \cite{Chen:2016:CMB:2946645.2946695}, given the number of edges in a network is usually significantly larger than the number of nodes. Utilizing the dependency in the influence probability across edges is expected to reduce the complexity.    



In this paper, we propose to model the dependency of activation probabilities on the edges for online influence maximization. Specifically, we assume each network node's influence and susceptibility are distinct and individually specified; and the activation probability on an edge is jointly determined by the giving node's influence and receiving node's susceptibility.    
This makes the activation probability matrix for the network low-rank. Then with such a low-rank structure, 
we propose a factorization based bandit solution to learn the latent \emph{influence factors} and \emph{susceptibility factors} on the nodes from the interactions with the environment. 
There are three important advantages of our factorization based bandit solution for online influence maximization. Firstly, it is able to capture the assortative mixing property of influence distribution in a network. Secondly, by directly learning the node-level parameters, the activation observation from one edge can be readily leveraged to other edges that share the same node. This reduces sample complexity for online learning. Third, comparing to existing online influence maximization solution with linear generation assumptions \cite{wen2017online, vaswani2017model, pmlr-v80-kalimeris18a}, our solution does not depend on the availability of manually constructed edge- or node-level features; it learns the property of network nodes via factorization. A rigorous theoretical analysis on the upper regret bound of the proposed solution is provided, where we prove considerable regret reduction comparing to solutions that model the activation probability on edges independently. Extensive empirical evaluations on two large-scale networks confirmed the effectiveness of our proposed solution.

\section{Related Work}

The problem of influence maximization has been extensively studied in offline settings \cite{kempe2003maximizing,chen2009efficient,chen2010scalable}, where the main focus is on computational efficiency in optimization. These solutions simply assume the influence model, i.e., the activation probabilities, can be specified by the network's properties (e.g., node degree), or by a transmission parameter that is specified as constant, random or drawn from a uniform distribution, or by estimations from logged propagation data \cite{bourigault2016representation,goyal2010learning, Netrapalli:2012:LGE:2318857.2254783,Saito:2008:PID:1430307.1430318}. However, such an overly simplified influence model ignores many important properties of real network influence patterns, e.g., assortativity \cite{newman2002assortative}.  
To conquer these limitations, online influence maximization has been studied under different assumptions and settings.
Chen et al. \cite {chen2013combinatorial} and Wang and Chen \cite{wang2017improving} formulated it as a combinatorial bandit problem and proposed an upper confidence bound \cite{auer2002using} based algorithm to estimate activation probabilities on a per-edge basis. Lei et al. \cite{lei2015online} studied a different objective of maximizing the number of unique activated nodes across multiple rounds of activation. In \cite{vaswani2015influence}, a bandit learning algorithm with both node-level (which nodes are eventually activated) and edge-level (who activated those nodes) feedback was developed. Methods similar to \cite{goyal2010learning, Netrapalli:2012:LGE:2318857.2254783,Saito:2008:PID:1430307.1430318} were used to map the node-level feedback back to edge-level activation probability estimation.  Wen et al. \cite{wen2017online} assumed the activation probability is a linear combination of edge-level features and an unknown global activation parameter, and proposed a linear contextual bandit solution. However, in practice it is very difficult to exhaustively specify the features for influence modeling on every edge, and many of those features are prohibited under privacy constraints, such as age, gender and martial status \cite{Aral337}. 

Our solution falls into another line of online influence maximization research that estimates influence parameters at the node level. 
Carpentier and Valko \cite{carpentier2016revealing} proposed a minimax optimal algorithm under a local model of the influence spread, where a source node can only activate its neighbors without considering influence cascade. Vaswani et al. \cite{vaswani2017model} proposed a diffusion model independent algorithm to estimate pairwise reachability between all pairs of nodes in a network. 
Olkhovskaya et al. \cite{olkhovskaya2018online} developed an algorithm for node-level feedback with a theoretical guarantee. But their algorithm is only designed for graphs with specific structures. 
In general, previous works in this line of research do not explicitly separate influence and susceptibility of nodes, which causes seriously degraded estimation of influence propagation, as reported in \cite{Aral2018}.
To the best of our knowledge, ours is the first  to model the structural dependency in edge activation probabilities for online influence maximization. We assume the activation probability of an edge is jointly determined by the influence of its giving node and the susceptibility of its receiving node. We estimate the influence and susceptibility parameters of nodes via edge-level activation feedback, without the need of edge features. This directly leads to reduced computational complexity in model estimation and reduced regret in influence maximization.


\section{Method}
In this section, we discuss our developed factorization-based bandit solution for online influence maximization. We first illustrate the problem setup, then introduce the details of our solution, and conclude this section by comparing our solution with existing algorithms to demonstrate its unique advantages.    

\subsection{Problem Setup}

\subsubsection{Influence maximization}
In an influence maximization problem, an input social network is modeled as a probabilistic directed graph $G = (\cV;\cE; \bP)$ with nodes $\cV$ over a set of users, edges $\cE$ for the set of directed connections between users, and activation probabilities $\bP = (p_{e_1}, p_{e_2}, ..., p_{e_{|\cE|}})$ where $p_e \in [0,1]$ represents the probability that the receiving node of edge $e$ will be activated by the giving node on $e$. 
We describe how influence propagates from nodes to their network neighbors with a stochastic diffusion model $D$. 
The influence spread of the selected seed node set $\cS$ with activation probability $\bP$ is the expected number of nodes activated by $\cS$ under the diffusion model $D$, denoted by $f_{D,\bP}(\cS)$. Given the graph $G$, the IM problem aims to find a particular set $\cS$ under a cardinality constraint $K$, which maximizes the influence spread $\cS^{opt}=\argmax_{|\cS|\leq K} f_{D, \bP}(\cS)$. This problem is known to be NP-hard~\cite{nemhauser1978analysis,kempe2003maximizing}, but it can be efficiently solved with approximations \cite{Tang:2014:IMN:2588555.2593670,Tang:2015:IMN:2723372.2723734,Nguyen:2016:SOS:2882903.2915207}. In this paper, such influence maximization algorithms are treated as an oracle, which takes a given graph, seed set size and activation probability as input, and outputs an appropriate set of seeds. Define ${\cS^*}=ORACLE(G,K,\bP)$ as the solution from the oracle. It serves as an $(\alpha,\gamma)$-approximation of $\cS^{opt}$, where $f_{D,\bP}(\cS^*)\geq \gamma f_{D,\bP}(\cS^{opt})$ with probability at least $\alpha$  \cite{chen2013combinatorial}. For example, we have $\gamma = 1-\frac{1}{e}-\varepsilon$ and $\alpha=1-1/|\cV|$ for most of influence maximization solutions \cite{kempe2003maximizing,chen2009efficient,chen2010scalable}, where $e$ is the base of natural logarithm and $\varepsilon$ depends on the accuracy of their Monte-Carlo estimate. 


\subsubsection{Online influence maximization with bandit}
Traditional works in influence maximization \cite{kempe2003maximizing,chen2009efficient,chen2010scalable} assume that in addition to the network structure, the algorithm either knows the per-edge activation probability or the probability can be specified from past propagation data \cite{goyal2010learning, Netrapalli:2012:LGE:2318857.2254783,Saito:2008:PID:1430307.1430318}. However, as we discussed before this assumption imposes various limitations for the practical use of influence maximization solutions; and it is preferred for the learner to estimate the activation probabilities by directly interacting with the network, i.e., online influence maximization. In each round of online influence maximization, the learner needs to choose seeds to 1) maximize influence spread (i.e., exploitation), and 2) improve its knowledge of the activation probabilities via feedback (i.e., exploration). Multi-armed bandit framework is thus a natural choice to handle the exploration-exploitation trade-off for online influence maximization \cite{vaswani2015influence,wen2017online}.


Formally, at each round $t$ during online influence maximization, the learner first chooses a seed node set ${\cS}_t \in \cV$ with cardinality $K$ by a predefined oracle based on its current knowledge from past observations: ${\cS}_t=ORACLE(G,K,\hat\bP_t)$, where $\hat\bP_t$ is the learner's current estimate of activation probability $\bP$. Influence then diffuses from the nodes in ${\cS}_t$ according to the diffusion model $D$. For example, under the Independent Cascade (IC) diffusion model, this can be interpreted as the environment generates a binary response $y_{e,t}$ by independently sampling $y_{e,t} \sim \text{Bern}(p_{e})$ for each concerned edge $e$ in the resulting cascade. Note the environment is always using ground-truth activation probability $\bP$. Then the learner receives the reward $f_{\text{IC}, \bP}({\cS}_t)$ which is the number of activated nodes by ${\cS}_t$. We should emphasize that for any edge $e \in \cE$, the learner observes the realization of $y_{e,t}$ if and only if its giving node is activated. To clarify, the active nodes deactivate at each time step, which makes it different from the setting in adaptive influence maximization \cite{golovin2011adaptive}.   With such edge-level feedback, the learner updates its estimate about the edge activation probability. The learner's objective is to maximize the expected cumulative reward over a finite $T$ rounds of interactions.

\subsection{A Factorization Bandit Solution}
In most of existing bandit-based online influence maximization solutions, the learner estimates the activation probabilities on all edges independently \cite{chen2013combinatorial,Chen:2016:CMB:2946645.2946695}, which cannot capture how the influence between nodes is formed and thus ignores the underlying structure behind the activation probabilities of edges. 
As suggested in a recent work \cite{Aral2018}, modeling nodes' influence and susceptibility is vital for influence maximization. In this work, we explicitly model this influence structure by assuming activation probability $p_{e} \in [0,1]$ on edge $e$ can be decomposed into two $d$-dimensional latent factors on the giving node and receiving node, i.e.,
\begin{align}  \label{eq:activation_assumption}
  p_{e}= \btheta_{g_e}^{\mt} \Beta_{r_e},
\end{align}
where $g_e$ and $r_e$ denote the giving node and receiving node on edge $e$ respectively. $\btheta_{v} \in \mathbb{R}^d$ is considered as the \emph{influence factor} on node $v$ and  $\Beta_{v} \in \mathbb{R}^d$ is the \emph{susceptibility factor} on node $v$.  This essentially imposes a low-rank structure assumption on the activation probability $\bP$ and considers the activation probabilities as a reflection of influence and susceptibility properties of the connected nodes. It naturally captures the assortativity of the influence network: users who are more influential, i.e., the nodes that have larger values on $\btheta_v$, trend to associate with larger activation probability on their neighboring node. The same for users who are more susceptible (e.g., measured by $\Beta_v$). 

In this work, we assume the diffusion follows the independent cascade model \cite{kempe2003maximizing}, and the edge-level feedback $y_{e,t}$ is observable to the learner. The estimation of activation probability $\bP$ can thus be obtained by an online factorization of the sequential observations about edge activations in the network. 
As a result, instead of learning the edge activation probabilities directly, the complexity of which is $O(|\cE|)$, our solution learns the underlying influence factor $\btheta_v$ and susceptibility factor $\Beta_v$ for all the nodes. This reduces the model complexity to $O(d|\cV|)$, which is considerably smaller than $O(|\cE|)$, especially in a large and dense graph.
To better illustrate our proposed learning solution, we introduce the following definitions about observed edges and observed nodes. 
\begin{definition}[Observed edge]
For any round $t$, a directed edge $e$ is considered as observed if and only if its start node is activated. 
 \end{definition}
Note that an edge is observed does not necessarily mean the edge is active (i.e., not necessarily $y_{e,t} =1$). With the observed edges, a set of nodes can be classified as observed, which is defined as follows, 
\begin{definition}[Observed node]
A node $v$ is observed if and only if at least one of its giving-neighbor nodes is active.
\end{definition}
Based on the above definition of observed edges and observed nodes, we can realize the relationship between observed edges and observed nodes: the giving and receiving nodes on an observed edges are both observed nodes. In this case, once an activation observation $y_{e,t}=1$ is obtained on edge $e$ at time $t$, statistics about the influence factor on its giving node $g_e$ and susceptibility factor on its receiving node $r_e$ can be updated. More importantly, all observed edges can be used for model update, which enables observation propagation to unobserved edges. We will discuss the benefit of this aspect with more details in Section \ref{sec_comparison}.

Due to the coupling between $\btheta$ and $\Beta$ imposed in the influence structure assumption in Eq \eqref{eq:activation_assumption}, we appeal to a coordinate decent algorithm built on matrix factorization to estimate the latent influence factor $\btheta_{v}$ and susceptibility factor $\Beta_{v}$ for each node $v \in \cV$. Specifically, the objective function for our factor estimation can be written as follows,
\begin{equation} \label{eq:obj_func}
   \min_{\{\btheta_{v},\Beta_{v}\}_{v \in \cV} }\sum^T_{t=1} \sum_{e \in \tilde \cE_{t}} ( \btheta_{g_e}^\mt \Beta_{r_e} -  y_{e,t} )^2  \!\!+ \lambda_1 \sum_{v \in \cV} \lVert \btheta_{v} \rVert_2 \!+ \lambda_2 \sum_{v \in \cV} \lVert \Beta_{v}  \rVert_2 
\end{equation}
where $\tilde \cE_t$ is the set of observed edges at time $t$, $\lambda_1$ and $\lambda_2$ are L2 regularization coefficients. The inclusion of L2 regularization term is critical to our solution in two folds. First, it makes the sub-problems in coordinate decent based optimization well-posed, so that we have closed form solutions for $\btheta_v$ and $\Beta_v$ at each round. Second, it helps to remove the scaling indeterminacy between the estimates of $\btheta_v$ and $\Beta_v$, and makes the $q$-linear
convergence rate of parameter estimation achievable \cite{LocalConvergenceALS, wang2016learning, wang2017factorization}.
The closed-form estimation of $\btheta_v$ and $\Beta_v$ with respect to Eq \eqref{eq:obj_func}
at round $t+1$ can be obtained by
 $\hat \btheta_{v,t+1} = \bA_{v,t+1}^{-1} \bb_{v,t+1}$ and $\hat\Beta_{v, t+1} = \bC_{v,t+1}^{-1} \bd_{v,t+1}$, in which,
\begin{align} \label{eq:model_update_stat1}
\bA_{v,t+1} &= \lambda_2 \bI + \sum_{i=1}^t \sum_{e \in \tilde \cE^{v, \text{giving}}_i } \hat \Beta_{r_e, t} \hat \Beta_{r_e, t}^\mt \\   \label{eq:model_update_stat2}
\bb_{v,t+1} &= \sum_{i=1}^t \sum_{e \in \tilde \cE^{v, \text{giving}}_i} \hat \Beta_{r_e, t} y_{e,i}\\ \label{eq:model_update_stat3}
\bC_{v,t+1} &= \lambda_1 \bI + \sum_{i =1}^t \sum_{e \in \tilde \cE^{v, \receiving}_i} \hat \btheta_{g_e,t} \hat \btheta^{\mt}_{g_e, t}\\ \label{eq:model_update_stat4}
\bd_{v,t+1} &= \sum_{i=1}^t \sum_{e \in \tilde \cE_i^{v,\receiving}} \hat \btheta_{g_e,t} y_{e,i}
\end{align}
$\tilde \cE_i^{v, \giving}$ is the set of observed edges at round $i$ where node $v$ is the giving node, and accordingly $\tilde \cE_i^{v,\receiving}$ is the set of observed edges where node $v$ is the receiving node. $\bI$ is a $d \times d$ identity matrix.  

The estimated influence factors $\{\hat\btheta_{v,t}\}_{v\in \cV}$ and susceptibility factors $\{\hat\Beta_{v,t}\}_{v\in \cV}$ give us an estimate of the edge activation probability by $\hat p_{e, t} = \hat \btheta^\mt_{g_e, t} \hat \Beta_{r_e, t}$ at round $t$. This represents the model's current best knowledge about influence propagation in graph $G$, and therefore serves for the exploitation purpose.
We also need a term to control exploration in online learning; and we appeal to the upper confidence bound (UCB) principle in this work \cite{Auer02}, as the confidence interval of our factor estimation is readily available. In our solution, the uncertainty of activation probability estimation during online update comes from the uncertainty on both latent factors, i.e., $\lVert \hat \btheta_{v,t} -\btheta_v^* \rVert$ and $\lVert \hat \Beta_{v,t} -\Beta_{v}^* \rVert$, where $\btheta_v^*$ and $\Beta_{v}^*$ are the ground-truth factors. Based on the closed form solution in our coordinate descend estimation, 
the estimation of activation probability $\hat p_{e,t}$ can be analytically bounded with a high probability as shown in the following lemma.

\begin{lemma}[Confidence bound of influence probability estimation] \label{lemma:activation_prob_CB}
For any node $v$, and directed edge $e$, $\delta \in (0,1), q \in (0,1)$, with a probability at least $1- \delta$, we have,

\begin{align*}
    \lVert \hat \btheta_{v,t} - \btheta_v^* \rVert_{\bA_{v,t}} &\leq \alpha_{v}^{\Beta}, 
   \lVert \hat \Beta_{v,t} - \Beta_v^*  \rVert_{\bC_{v,t}} \leq \alpha_{v}^{\btheta} \\
    |p^*_{e} -\hat p_{e, t}| &\leq B^{\btheta}_{g_e,t}  + B^{\Beta}_{r_e,t} + 2q^{2t}
\end{align*}
where $B^{\btheta}_{v,t} =  \alpha_v^{\btheta}  \lVert \hat\btheta_{v,t} \rVert_{\cC_{v,t}^{-1}}$, $B^{\Beta}_{v,t} =  \alpha_v^{\Beta} \lVert \hat \Beta_{v,t} \rVert_{\bA^{-1}_{v},t}$,  $\alpha_v^{\btheta} = \sqrt{ \log \frac{\text{det}(\bC_{v,T})}{\delta^2 \text{det}(\lambda_2 \bI) } } + \frac{\lambda_2(1-q) + 2q}{\sqrt{\lambda_2}(1-q)}$, and $\alpha_v^{\Beta} = \sqrt{ \log \frac{\text{det}(\bA_{v,T})}{\delta^2 \text{det}(\lambda_1 \bI) } } + \frac{\lambda_1(1-q) + 2q}{\sqrt{\lambda_1}(1-q)}$.
\end{lemma}

This lemma provides a tight high probability upper bound of the activation probability estimation on each edge, which enables us to perform efficient UCB-based exploration for model update. 
Denote the resulting upper confidence bound as $\bar p'_{e,t}$, with $\bar p'_{e,t} = \hat\btheta_{g_e,t}^\mt \hat \Beta_{r_e,t} + \text{CB}_{e,t}$ and $\text{CB}_{e,t} = B^{\btheta}_{g_e,t}  + B^{\Beta}_{r_e,t} + 2q^{2t}$. As this upper bound might not be in a valid value as a probability, we use a projection operation to map it to the range of $[0, 1]$ and get $\bar p_{e,t} = \text{Proj}_{[0,1]} (\bar p'_{e,t})$.  $\text{Proj}_{[0,1]}(x)$ is a function projecting its real value input $x$ to the range of $[0,1]$. More specifically, if the input $x$ is negative, the output is $0$; if $x$ is larger than $1$, the output is $1$; otherwise $x$ will be returned. 
Note that $\bar p'_{e,t}$ is a high probability upper bound of $p_{e,t}$, i.e., $p_{e,t} \leq \bar p'_{e,t}$. After this projection operation we can still guarantee that with a high probability $p_{e,t} \leq \bar p_{e,t}$. This is true because: 1). When $\bar p'_{e,t} > 1$, $\bar p_{e,t} =1$, which naturally guarantees $p_{e,t} \leq \bar p_{e,t}$; 2). When $\bar p'_{e,t} < 0 $, $\bar p_{e,t} = 0$, we have $p_{e,t} \leq  \bar p'_{e,t} < \bar p_{e,t}$.  
As an upper confidence bound of activation probability $\bar p_{e,t}$, it well balances the need of exploration and the need of exploitation \cite{Auer02}. We thus directly feed $\bar \bP_t = (\bar p_{e_1,t}, \bar p_{e_2,t}, ...., \bar p_{e_{|\cE|},t})$ at round $t$ into the ORACLE to solve one round of influence maximization, from which edge-level feedback on the observed edges will be obtained. Based on the observed edges, we get a set of observed nodes, and node-level influence factor $\btheta_v$ and susceptibility factor $\Beta_v$ and their associated statistics can be updated based on their closed form solutions. We summarize the details of our factorization based online influence maximization algorithm in  Algorithm \ref{alg:IMFB}, and name it as Influence Maximization with Factorization-Bandits, or IMFB in short. 

\begin{algorithm}
\caption{Influence Maximization with Factorization-Bandits (IMFB)}\label{alg:IMFB}
\begin{algorithmic}[1]
\State \textbf{Inputs:} Graph $G$ and oracle $\textbf{ORACLE}$, latent factor dimension $d$, regularization parameter $\lambda_1$ and $\lambda_2$, $q \in (0,1)$
\State \textbf{Initialization:}  $\bA_{v,1} = \lambda_1 \bI, \bC_{v,1} = \lambda_2 \bI$, sample $\hat \btheta_{v,1} \in \mathbf{R}^{d}$, $\hat \Beta_{v,1} \in \mathbf{R}^{d}$ for $v \in \cV$
\For{$t = 1, 2, \ldots,T$}
\State Compute $\bar p_{e,t} = \text{Proj}_{[0,1]}\big( \hat \btheta_{g_e,t}^\mt \hat \Beta_{r_e,t}  + \text{CB}_{e,t} \big)$ with $\text{CB}_{e,t} = \alpha^{\Beta}_{g_e} \lVert \hat \Beta_{g_e,t} \rVert_{\bA^{-1}_{g_e,t}} +  \alpha^{\btheta}_{r_e} \lVert \hat \btheta_{r_e,t} \rVert_{\bC^{-1}_{r_e,t}} + 2q^t$ for $e \in \cE$ 
\State Choose $\cS_t = \text{ORACLE}(G,K,\bar \bP_t)$ in which $\bar \bP_{t} = \{ \bar p_{e,t} \}_{e \in \cE}$
\State Observe the edge-level feedback $\{ y_{e,t} \}_{e \in \tilde \cE_t}$ in which $\tilde \cE_t$ is the observed edges.
\State Update the node-level influence statistics and susceptibility statistics on the observed nodes according to Eq \eqref{eq:model_update_stat1},\eqref{eq:model_update_stat2},\eqref{eq:model_update_stat3},\eqref{eq:model_update_stat4} 
\EndFor
\end{algorithmic}
\end{algorithm}

\subsection{Comparison with Existing Solutions}
\label{sec_comparison}
\subsubsection{Node-level v.s., edge-level influence model estimation}
In our proposed solution, we directly estimate the node-level influence and susceptibility factors by performing online factorization based on sequential edge-level observations. In this subsection, we illustrate the advantages of such a node-level parameter estimation by comparing with solutions that directly estimate edge-level parameters. One classic online influence maximization solution with edge-level parameter estimation is the CUCB algorithm \cite{wang2017improving}.
Comparing to CUCB, there are two important advantages in IMFB. First, reduced model complexity. In CUCB the edge activation probabilities are estimated independently across the edges, the complexity of which is $O(|\cE|)$ (and it can become as large as $O(|\cV|^2)$ in the complete graph case). While the model complexity in IMFB is $O(d|\cV|)$, which is considerably smaller, as the dimension of latent space should be much smaller than the number of nodes. 
Second, reduced sample complexity. Figure \ref{fig:illu} shows an example of influence model update in CUCB and IMFB after one round of influence maximization. Suppose node $v_0$ is activated at this round, according to the definition of observed edges, $e_{v_0,v_1}$ and $e_{v_0,v_2}$ (edges with solid lines) are observed edges, and $e_{v_3,v_1}$, $e_{v_4,v_1}$, and $e_{v_5,v_1}$ (edges shown with dash lines) are unobserved edges. 
In CUCB, shown in Figure \ref{fig:illu} (a), its activation probability statistics will only be updated for those observed edges and no information is learned for those unobserved edges. In IMFB, by utilizing the fact that the activation on a particular edge is a reflection of the giving node's influence and the receiving node's susceptibility, observations from the observed edges can be propagated to the unobserved edges, as shown in Figure \ref{fig:illu} (b). Specifically, because $\hat\Beta_{v_1}$ will be updated by this round's observations, the estimation of $\hat p_{e_{v_3,v_1}}$, $\hat p_{e_{v_4,v_1}}$ and $\hat p_{e_{v_5,v_1}}$ can all be improved for the next round of interaction. This directly accelerates the learning of influence structure. This advantage becomes more obvious and significant when the network is denser. Our theoretical analysis in the following section proves that when the network is fully connected, we can reduce the regret of IMFB comparing to CUCB by a factor of $\sqrt{|\cV|}$. In addition, as we discussed earlier, independent edge-level estimations cannot guarantee to capture the important properties of influence propagation in real networks, e.g., assortativity. IMFB explicitly encode such structure into model learning, which can better realize the structural dependency in influence propagation.

\begin{figure}[t]
\centering
\setlength\tabcolsep{2pt}
\begin{tabular}{ >{\centering\arraybackslash}m{4.2cm} >{\centering\arraybackslash}m{4.2cm}}
\hspace*{-3mm}
\includegraphics[width=3.3cm]{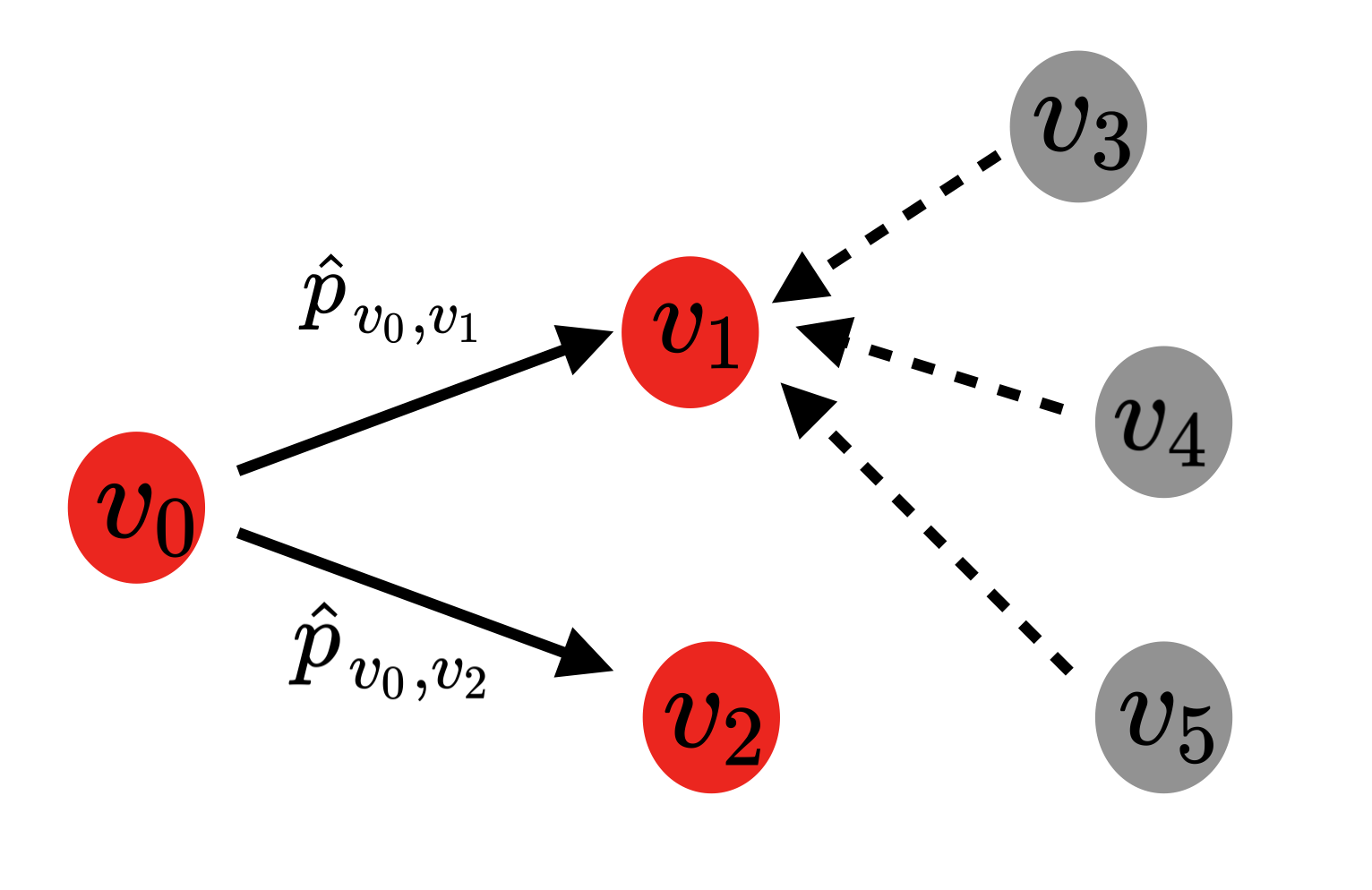} &
\includegraphics[width=3.5cm]{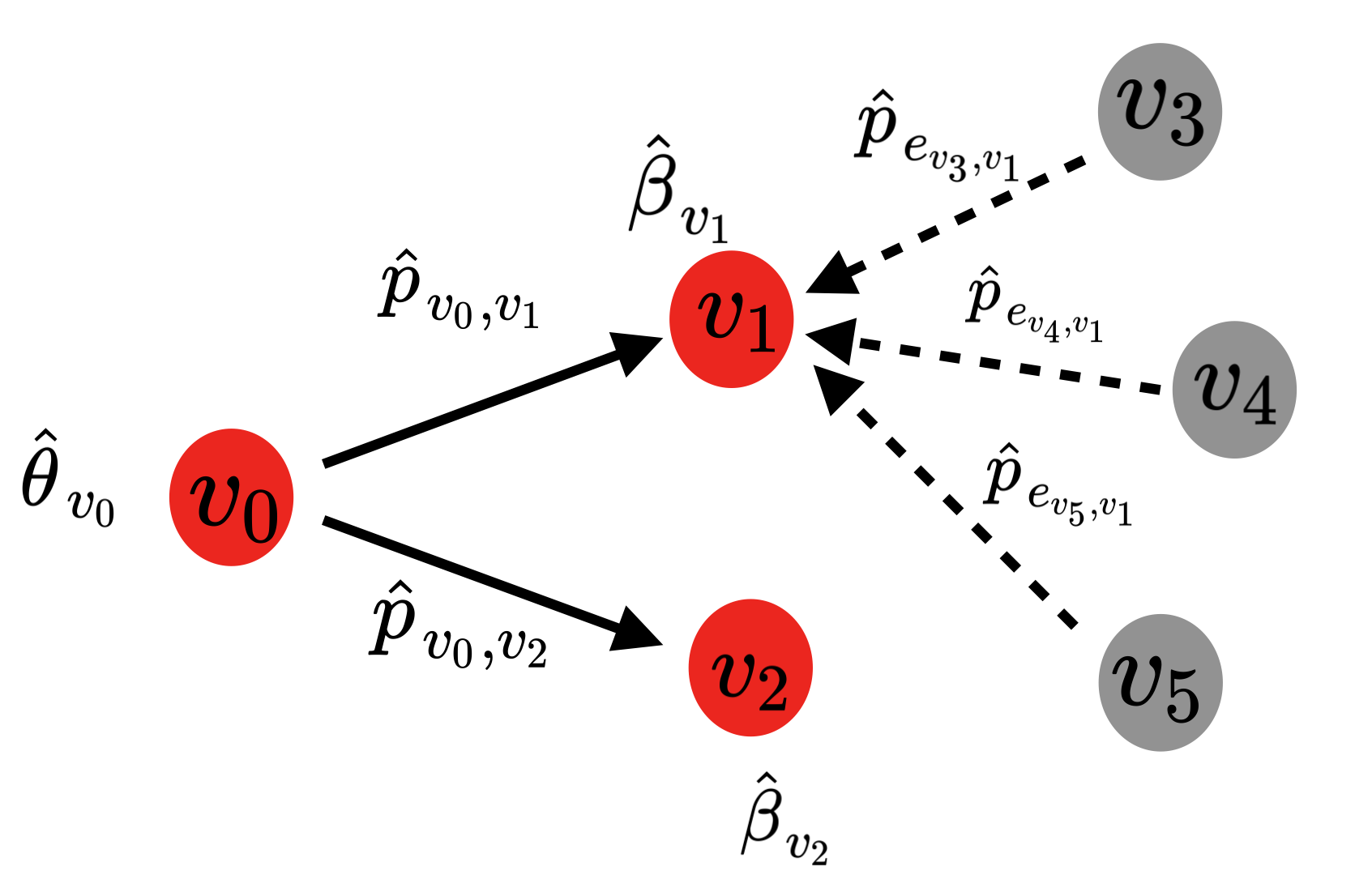} \\
(a) CUCB update &  (b) IMFB update  \\
\end{tabular}
\vspace{-2mm}
\caption{Comparison between CUCB and IMFB in influence model update.} \label{fig:illu}
\vspace{-4mm}
\end{figure}


\subsubsection{Activation probability v.s., reachability}
A recent work, DILinUCB \cite{vaswani2017model}, also obtains an $O(d|\cV|)$ model complexity by performing node-level parameter estimation. But DILinUCB is conceptually very different from our IMFB. Instead of estimating the activation probability on edges from the interactions, DILinUCB directly estimates the reachability between all pairs of nodes (no matter if they are directly connected or not) from edge-level  feedback via a linear contextual bandit model. There are two important limitations of it.
Firstly, as mentioned in their paper, the proposed reachability based surrogate objective is only an heuristic approximation to the original online influence maximization problem. There is no theoretical guarantee about its approximation quality. Secondly, although its claimed diffusion model independent, their linear reachability assumption may violate the influence diffusing rules in some diffusion models, for example the independent cascading diffusion model. 

\section{Regret Analysis}
Let $\cS^{opt}$ be the optimal solution on an input graph, and $\cS^* = \text{ORACLE}(G, K, \bP^*)$ be the solution from an oracle. Here we use $\bP^*$ to denote the ground-truth activation probability in the network. For any $\alpha, \gamma \in [0,1]$, we say that $\text{ORACLE}$ is an $(\alpha, \gamma)$-approximation oracle for a given $(G,K)$ if for any activation probability vector $\bP$, $f_{\bP}(\cS^*) \geq \gamma f_{\bP}(\cS^{opt})$ with probability at least $\alpha$. 
In this section, we simplify the notation of reward function $f_{D, \bP}( \cS )$ to $f_{\bP}( \cS )$, since we are consistently using independent cascade model. 

To analyze the performance of online influence maximization algorithms with an approximation/randomized oracle, we look at an approximate performance metric -- the scaled cumulative regret: $R^{\alpha \gamma}(T) = \sum_{i=1}^T \bbE[R_i^{ \alpha \gamma}]$, in which $ \alpha \gamma >0$ is the scale and $R_t^{\alpha \gamma} = f_{\bP^*}(\cS^{opt}) - \frac{1}{\alpha \gamma} f_{\bP^*}(\cS_t)$. Based on this regret definition, we have Theorem \ref{theorem:main_regretbound} specifying the regret upper bound of IMFB, which utilized two important proprieties of the reward function $f_{\bP}(\cS)$ in influence maximization problem: monotonicity and 1-Norm bound smoothness \cite{wang2017improving}.

\begin{theorem}
\label{theorem:main_regretbound}
Assuming that the ORACLE is an $(\alpha,\gamma)$-approximation algorithm, we have the following upper regret bound in IMFB,
\begin{align}
    R^{\alpha\gamma}(T)  \leq O\big( \frac{d B}{\alpha \gamma} & ( \sqrt{T |\cV^{\giving}| |\cV| D^{\textout} }  \log (TD^{\textout})\\ \nonumber
    & + \sqrt{T |\cV^{\receiving}| |\cV| D^{\textin} } \log (TD^{\textin})) \big)
\end{align}
in which $B$ is the bounded smoothness constant and it can be proved that $B \leq |\cV|$ in influence maximization problem \cite{wang2017improving}. $D^{\textout}$ and $D^{\textin}$ are the maximum out-degree and in-degree of graph $G$. $|\cV^{\giving}|$ and $|\cV^{\receiving}|$ are the maximum number of observed giving nodes and receiving nodes at any time point during the interactions.
\end{theorem}
\textbf{Proof sketch of Theorem \ref{theorem:main_regretbound}}: According to the regret definition, the expected regret at time $t$ can be bounded by,
\begin{align*}
    \bbE[R^{\alpha \gamma}_t]  \leq \frac{1}{\alpha \gamma} \bbE[f_{\bP^*}(\cS^*) - f_{\bP^*}(\cS_t)]
     & \leq  \bbE[f_{\bar \bP_t}(\cS_t) - f_{\bP^*}(\cS_t)] \\ \nonumber
    & \leq B \sum_{e \in  \tilde \cE_t} |\bar p_{e,t} -  p_{e}^{*}|
\end{align*}
in which the second inequality is based on the monotonicity of the reward function and the ORACLE's seed node selection. The third inequality is based on the 1-Norm bounded smoothness of the reward function and $B$ is the bounded smoothness constant \cite{wang2017improving}. Then combining the activation probability upper bound provided in Lemma \ref{lemma:activation_prob_CB}, the conclusion in Theorem \ref{theorem:main_regretbound} can be obtained. Due to space limit, detailed proof of this theorem and related technical lemmas will be provided in the supplement of this paper. 

\textbf{Analysis of Theorem \ref{theorem:main_regretbound}}: Since $|\cV^{\receiving}|$, $|\cV^{\giving}|$, $D^{\textout}$, and $D^{\textin}$ can all be simply upper bounded by $|\cV|$, the worst case upper regret bound of IMFB is $O(d|\cV|^{\frac{5}{2}} \sqrt{T})$, which is still better than the upper regret bound most online influence maximization algorithms: in a complete graph, their upper regret bound for IMLinUCB \cite{wen2017online} and CUCB \cite{wang2017improving} are $O(d|\cV|^{3} \sqrt{T})$ and $O(|\cV|^{3} \sqrt{T})$ respectively. 


\section{Experiments}

\begin{figure*}[ht]
\centering
\vspace{-2mm}
\setlength\tabcolsep{4pt}
\begin{tabular}{ >{\centering\arraybackslash}m{6.0cm} >{\centering\arraybackslash}m{6.0cm} >{\centering\arraybackslash}m{6.0cm}}
\hspace*{-3mm}
\includegraphics[width=5.6cm]{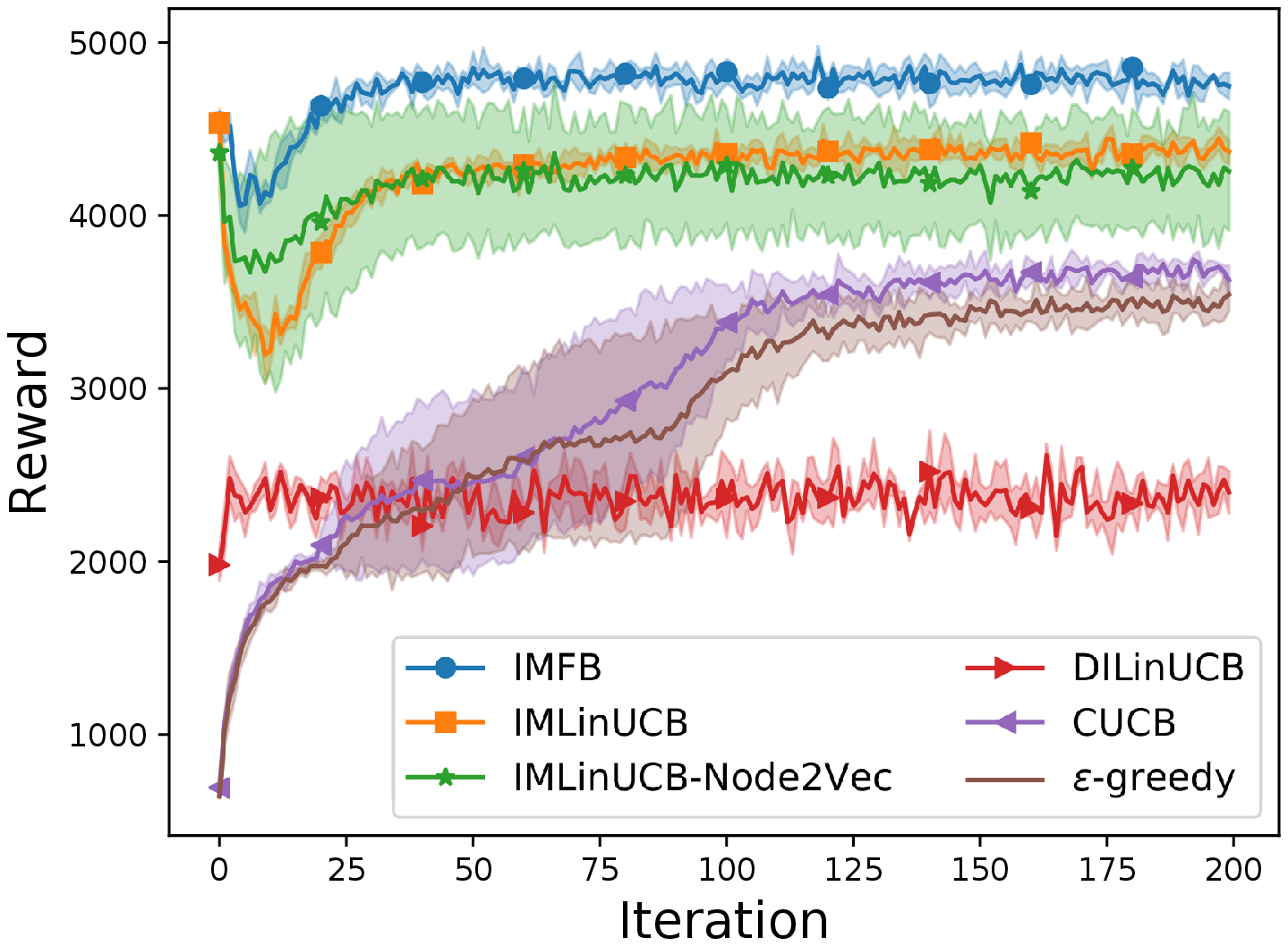} &
\includegraphics[width=5.6cm]{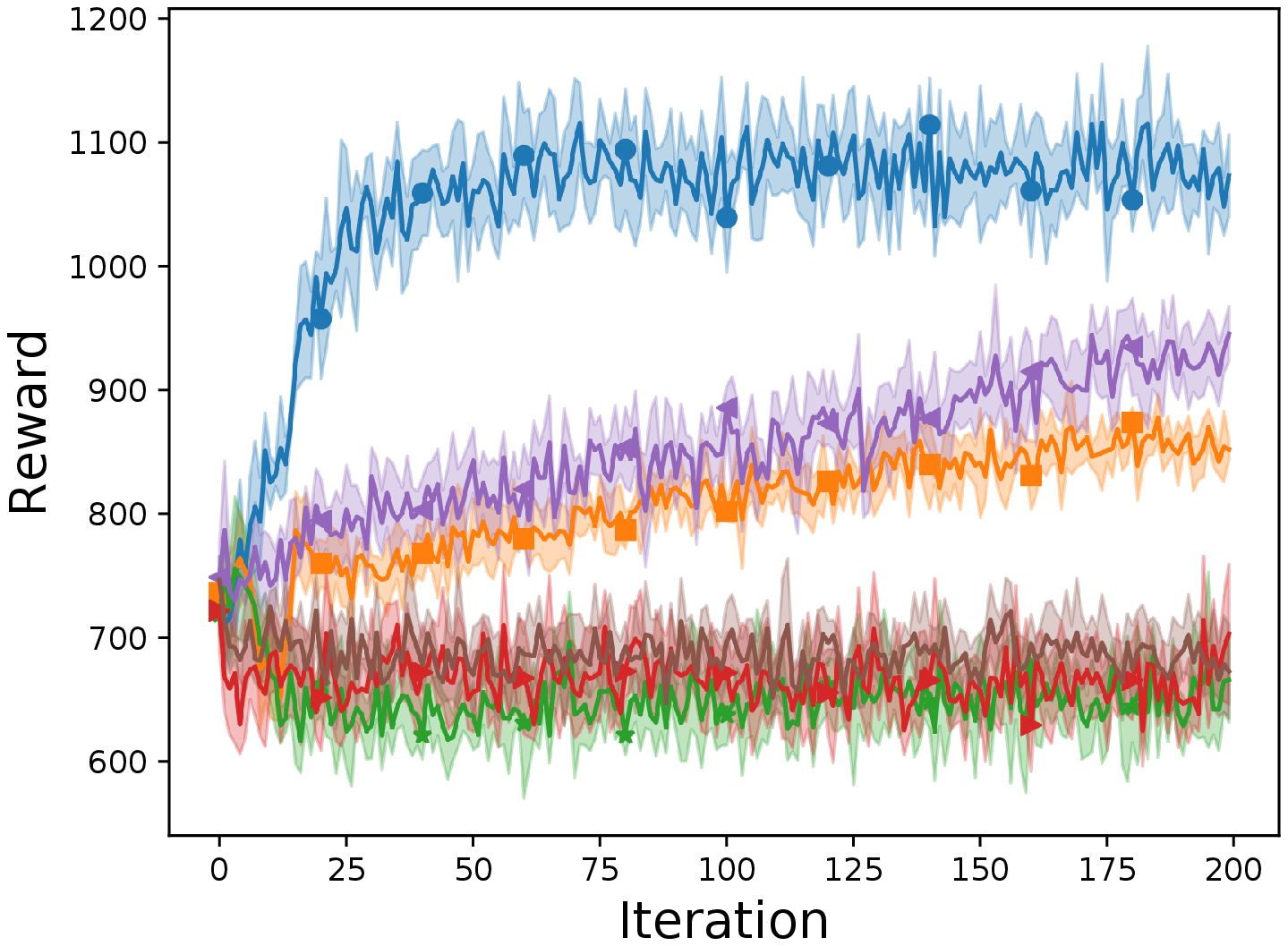} &
\includegraphics[width=5.6cm]{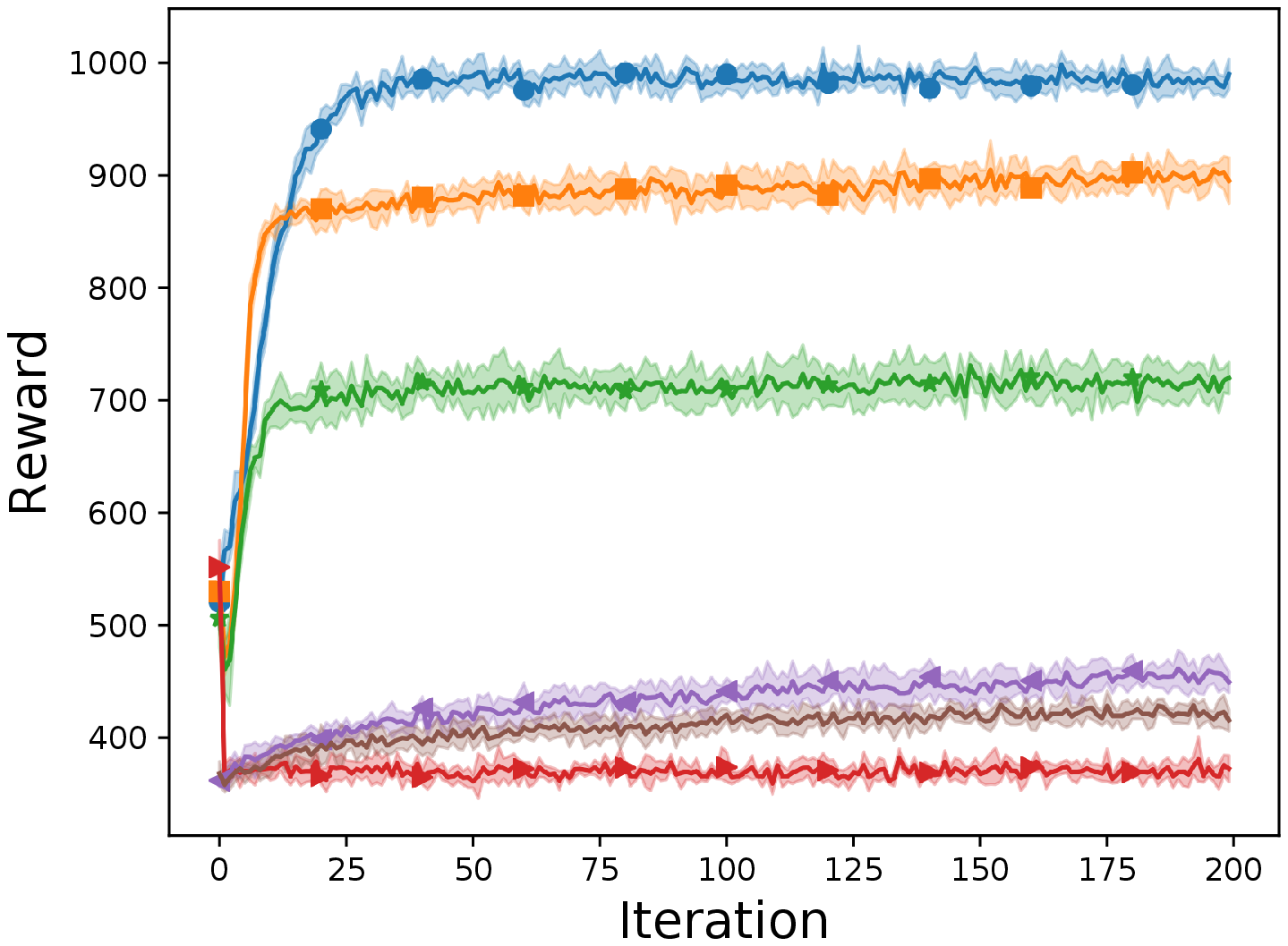} \\
(a) NetHEPT  &  (b) Flickr  & (c) Flickr with controlled nodes \\
\end{tabular}
\vspace{-2mm}
\caption{Real-time average reward from online influence maximization.}  \label{fig:avg}
\vspace{-4mm}
\end{figure*}

In this section, we performed extensive empirical evaluations on two real-world networks and compare the proposed solution with a set of state-of-the-art  online influence maximization baselines. 

\subsection{Datasets}
We evaluate our proposed IMFB algorithm on two real-world datasets: NetHEPT\footnote{\url{https://snap.stanford.edu/data/cit-HepTh.html}} (with 27,770 nodes, and 352,807 )  and Flickr\footnote{\url{https://snap.stanford.edu/data/web-flickr.html}} (with 12,812 nodes and 137,986 edges). 
On the NetHEPT dataset, the ground-truth influence factors and susceptibility factors on the nodes are sampled from a uniform distribution $U(0,0.1)$ and then normalized by their L2 norms. The activation probabilities on the edges are generated according to Eq \eqref{eq:activation_assumption}. The average edge activation probability on this resulting dataset is 0.053.  The Flickr dataset has a very skewed degree distribution, which makes it a very easy task for all the online influence maximization algorithms if the influence factors and susceptibility factors of all the nodes are sampled from the same distribution. In order to make the learning task on this dataset more challenging, we need to make the `soft degree' distribution on this network more balanced. Here we define a node's `soft degree' as the summation of its activation probability on all its connected out-going edges. In order to achieve this, we sample influence factors and susceptibility factors from 10 uniform distributions with different expectations, including $U(0,0.1), U(0.1, 0.2), ..., U(0.9, 1.0)$, and then normalize them using  their L2 norms respectively. We group the nodes into 10 groups according to their `hard degree' (here a node's original degree is referred as `hard degree'). We assign the groups of nodes with higher hard degree to the distribution that has a lower expectation, such that we have a much more balanced soft degree distribution on this network. This makes the learning of activation probabilities more challenging. The average edge activation probability on this resulting dataset is 0.065.

In all following experiments, the dimensionality for the latent influence and susceptibility factor is set to 20, and seed node size $K$ is set to 300 by default.

\subsection{Baseline Algorithms}
We empirically evaluate our proposed IMFB algorithm against the following state-of-the-art online IM algorithms.

\noindent $\bullet$ \textbf{\textit{CUCB}}  \cite{chen2013combinatorial,wang2017improving}: it estimates the activation probabilities on the edges independently and uses upper confidence bound of the estimation to balance exploitation and exploration. 

\noindent $\bullet$ \textbf{\textit{$\epsilon$-greedy}}: it estimates the activation probabilities on the edges independently, and uses $\epsilon$-greedy \cite{sutton1998introduction} to balance exploitation and exploration. 


\noindent $\bullet$ \textbf{\textit{IMLinUCB}} \cite{wen2017online}: it learns an edge-level bandit model and estimates the activation probabilities with given edge-level features. In our experiment, we test two variants of IMLinUCB depending on how the edge-level features are constructed: (1). 
We directly use the ground-truth influence and susceptibility factors $\btheta^*$ and $\Beta^*$ to generate a set of high-quality features for it. For each edge $e$, we take outer product on $\btheta_{g_e}^*$ and $\Beta_{r_e}^*$ and reshape it to a $d^2$-dimensional vector. We need to emphasize that since $\btheta_{g_e}^*$ and $\Beta_{r_e}^*$ are directly used to generate the ground-truth activation probability, this feature setting can lead to the best possible version of IMLinUCB, i.e., it only needs to recognize the diagonal terms in the outer product. (2). We follow the setting proposed in \cite{wen2017online}, where we first use node2vec \cite{grover2016node2vec} to create node embedding vectors and use an element-wise product between two nodes of an edge to get the edge features. This variant of IMLinUCB is referred to as IMLinUCB-Node2Vec in the following discussion. 

\noindent $\bullet$ \textbf{\textit{DILinUCB}} \cite{vaswani2017model}: is a model-independent contextual bandit algorithm for online influence maximization. DILinUCB requires features on the receiving nodes as input to estimate the reachibility between node pairs. We provide ground-truth susceptibility factor $\Beta_v^*$ for $v \in \cV$ as input of node features to DILinUCB. 

For IMFB and all baselines except DILinUCB, the DegreeDiscountIC algorithm proposed in ~\cite{chen2009efficient} is used as the influence maximization oracle. For DILinUCB, since the original influence maximization objective function is replaced with a heuristic surrogate function, the greedy oracle suggested in their solution \cite{vaswani2017model} is used here. In all our experiments, we run all algorithms over 200 rounds and report the averaged results from 5 independent runs.

\subsection{Comparison on Influence Maximization}
We first compare different algorithms' performance in maximizing influence during online learning. 
For better visibility, all subfigures in Figure \ref{fig:avg} share the same set of legend. 

In Figure \ref{fig:avg} (a) \& (b), we report the real-time average reward and variance (showed by the shaded area) of the collected reward from different algorithms on NetHEPT and Flickr dataset respectively. From this result, we first observe that IMFB shows significant improvement over all the baselines on both datasets. We also observe that on both datasets DILinUCB performs consistently the worst, which is expected because of the heuristic reachability assumption. 

On the NetHEPT dataset shown in Figure \ref{fig:avg} (a), we can observe that the learning process of CUCB and $\epsilon$-greedy are much slower than the other solutions. This observation verified our claim about the necessity of explicitly modeling the influence structure: when the edge activation probabilities are independently modeled (e.g., in CUCB and $\epsilon$-greedy), the learner can only update and improve its estimation quality on the observed edges. This makes the learning very slow, especially on this relative denser dataset. Although uncertainty of the estimation can help avoid local optimal, it may also take time to recognize the globally optimal seed nodes, especially when a lot of edges need exploration. 
While for IMFB and both variants of IMLinUCB, since they can propagate observations through node-level parameter estimation (in IMFB) or shared edge parameter estimation, the learning process is much faster. 

On the Flickr dataset shown in Figure \ref{fig:avg} (b), IMLinUCB performs even worse than CUCB. The potential reason for this bad performance is that IMLinUCB is over-parameterized, which slows down its learning convergence. For example, among all provided $d^2$ features, only $d$ of them are effective. The algorithm needs a lot more observations to recognize the relevance of those features. While at the same time, since the dataset is relatively sparser, i.e., the number of observed edges are relatively smaller in each round of interaction, which makes IMLinUCB's learning process even slower. 
In addition, we can also notice that IMLinUCB-Node2Vec is one of the worst performing algorithms. This indicates IMLinUCB is very sensitive to the quality of manually constructed edge-features. 

Based on our observations in Figure \ref{fig:avg} (a) \& (b), we designed a new setting on the Flickr dataset to better illustrate the importance and advantages of node-level parameter estimation in our proposed IMFB. On the Flickr dataset, we first categorize the nodes into two types: Type 1 includes nodes that have very high hard degrees; and Type 2 includes nodes that have relatively low hard degrees. Then we controlled the generation of influence and susceptibility factors to make the out-going edges from Type 1 nodes associated with extremely low activation probabilities. The purpose of this design is to make the average soft degree in Type 1 nodes small. And for Type 2 nodes, we control the activation probability generation to make their associated edges' activation probability extremely large. And we remove some of the edges (but not all of them) connecting these two types of nodes to further control the degree distribution. 

At the beginning of online influence maximization, when a learner's estimation about activation probabilities is not accurate, hard degree plays a dominating role in the ORACLE's decision. Hence, on this manipulated Flickr network, the ORACLE trends to start with Type 1 nodes (because their hard degree is very high) but the optimal seed nodes are more likely to concentrate on Type 2 nodes, where the average activation probability is higher. In this case, if there is no observation prorogation among edges, it would take more rounds of interactions for a learner to realize that Type 2 nodes are actually better. In IMFB, even though at the beginning stage it would also start from Type 1 nodes, because of its node-level parameter estimation, observations from observed edges can be effectively propagated to those unobserved edges. If those unobserved edges are connected to Type 2 nodes (because we did not remove all the edges connecting Type 1 and Type 2 nodes), estimation about the Type 2 nodes and their connected edges can thus be improved. This observation propagation to unobserved edges helps IMFB learn much faster than the other baselines.  This benefit is also illustrated in Figure \ref{fig:illu} of Section \ref{sec_comparison}. Empirical evaluation of average reward under this setting can be found in Figure \ref{fig:avg} (c), which verified our claimed benefit on IMFB: On this dataset, IMFB obtained much more significant improvement: it obtained $230 \%$ improvement over CUCB and $\epsilon$-greedy. 

\begin{table*}[ht]  
    \centering
    \vspace{-1mm}
    \caption{Robustness to dimensionality misspecification.} 
\label{tab:robustness_dimension}
    \vspace{-3mm}
\begin{tabular*}{\textwidth}{@{\extracolsep{\fill} } c|rrrr|rrrrr} %
\toprule
  &IMFB-1 &IMFB-5 & IMFB-20  &IMFB-60  & $\epsilon$-greedy & CUCB & IMLinUCB-Node2Vec & IMLinUCB & DILinUCB \\
\midrule 
 NetHEPT & \textbf{4657.80} & \textbf{4711.30} & \textbf{4731.90} & \textbf{4715.80} & 2882.10 & 3022.30 & 4262.40 & 4392.50 & 2368.80
\\
Flickr & \textbf{979.80} & \textbf{1044.88} & \textbf{1054.56} & \textbf{1044.17} & 706.14 & 880.37& 744.52 & 800.86 & 715.89
\\
\bottomrule
\end{tabular*}
\vspace{-2mm}   
\end{table*}

\begin{table*}[ht] 
    \centering
    \vspace{-1mm}
    \caption{Reward comparison under different environment settings on Flickr dataset.} 
    \label{tab:robust}
    \vspace{-2mm}
\begin{tabular*}{\textwidth}{@{\extracolsep{\fill} } l|rrrrrr}
\toprule
&IMFB & $\epsilon$-greedy & CUCB & IMLinUCB-Node2Vec & IMLinUCB & DILinUCB \\
\hline
$K=50 $& \textbf{203.92}& 161.80& 194.71& 130.47& 179.24
 &131.31\\
$K=100 $& \textbf{405.50}& 295.09& 352.51& 243.68 & 332.28 & 257.73
\\
$K=300 $& \textbf{1054.56}& 706.14 &880.37& 744.52 &800.86 & 715.89
\\
$K=1,000 $& \textbf{2737.86}& 2224.10& 2508.25& 2013.49& 2157.98 & 2198.64
\\
 $\eta=0 $& \textbf{1054.56}& 706.14 &880.37& 744.52 &800.86 & 715.89
\\
 $\eta \sim U(-0.05, 0.05) $& \textbf{1455.83} &1120.56 &1241.26 &1205.89 &1267.67 & 1131.39
\\
$\eta \sim U(-0.1, 0.1)  $&\textbf{1956.80} &1641.30 &1740.00 &1856.83 &1921.10 & 1711.80
\\
$\eta \sim U(-0.3, 0.3)  $& \textbf{2766.07} &2587.10 &2690.30 &2659.34 &2736.78 & 2590.18
\\
\hline
$c=0.5$& \textbf{432.45}&359.19 &367.75 &383.46 &412.45 &350.28
\\
$c=1.0$&\textbf{1054.56} &706.14 &880.37 &744.52 &800.86 &715.89
\\
$c=1.5$&\textbf{2727.07} &1863.94 & 2407.70 & 2035.44 & 2479.88 & 2278.18
\\
$c=2.0$&\textbf{3884.68} &3517.8 &3776.86 &3598.31 &3658.57 & 3526.53
\\
\bottomrule
\end{tabular*}
\vspace{-2mm}
\end{table*}

\subsection{Influence Model Estimation Quality}
In order to better understand the algorithms' effectiveness in activation probability learning, we compared the estimation quality of IMFB with related baselines on both datasets in Figure \ref{fig:loss} with solid lines (using the right-side of y-axis). Specifically, we calculated the absolute difference between the estimated activation probability $\hat p_{e,t}$ and ground-truth probability $p_e^*$ over \emph{observed edges}. We omit $\epsilon-$greedy and DILinUCB here since $\epsilon-$greedy's performance is very similar to CUCB, and DILinUCB does not directly estimate the activation probabilities. 

\begin{figure}[th!]
\centering
\vspace{-1mm}
\setlength\tabcolsep{4pt}
\begin{tabular}{ >{\centering\arraybackslash}m{4.2cm} >{\centering\arraybackslash}m{4.2cm}}
\hspace*{-2mm}
\includegraphics[width=4.4cm]{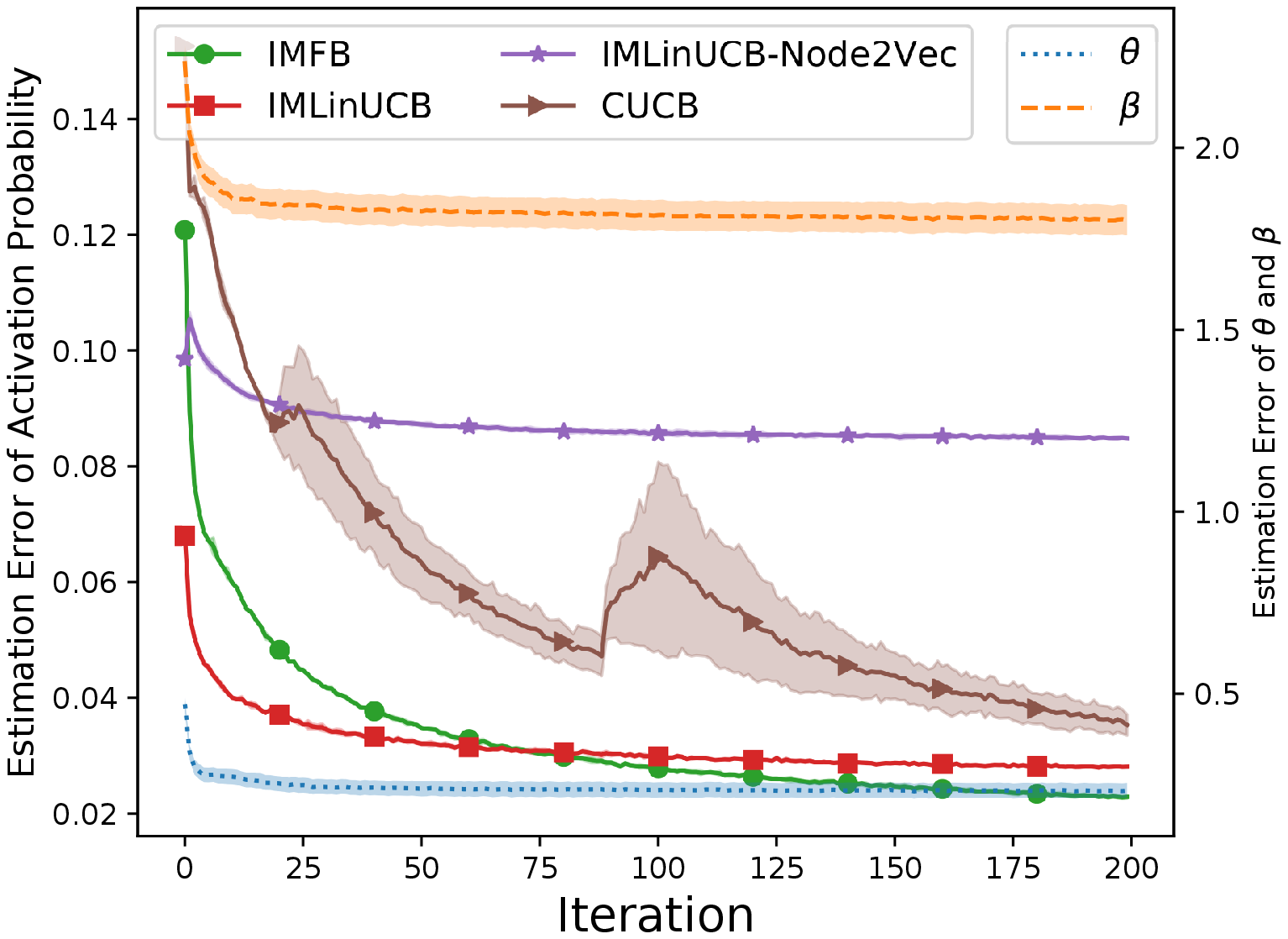} &
\includegraphics[width=4.4cm]{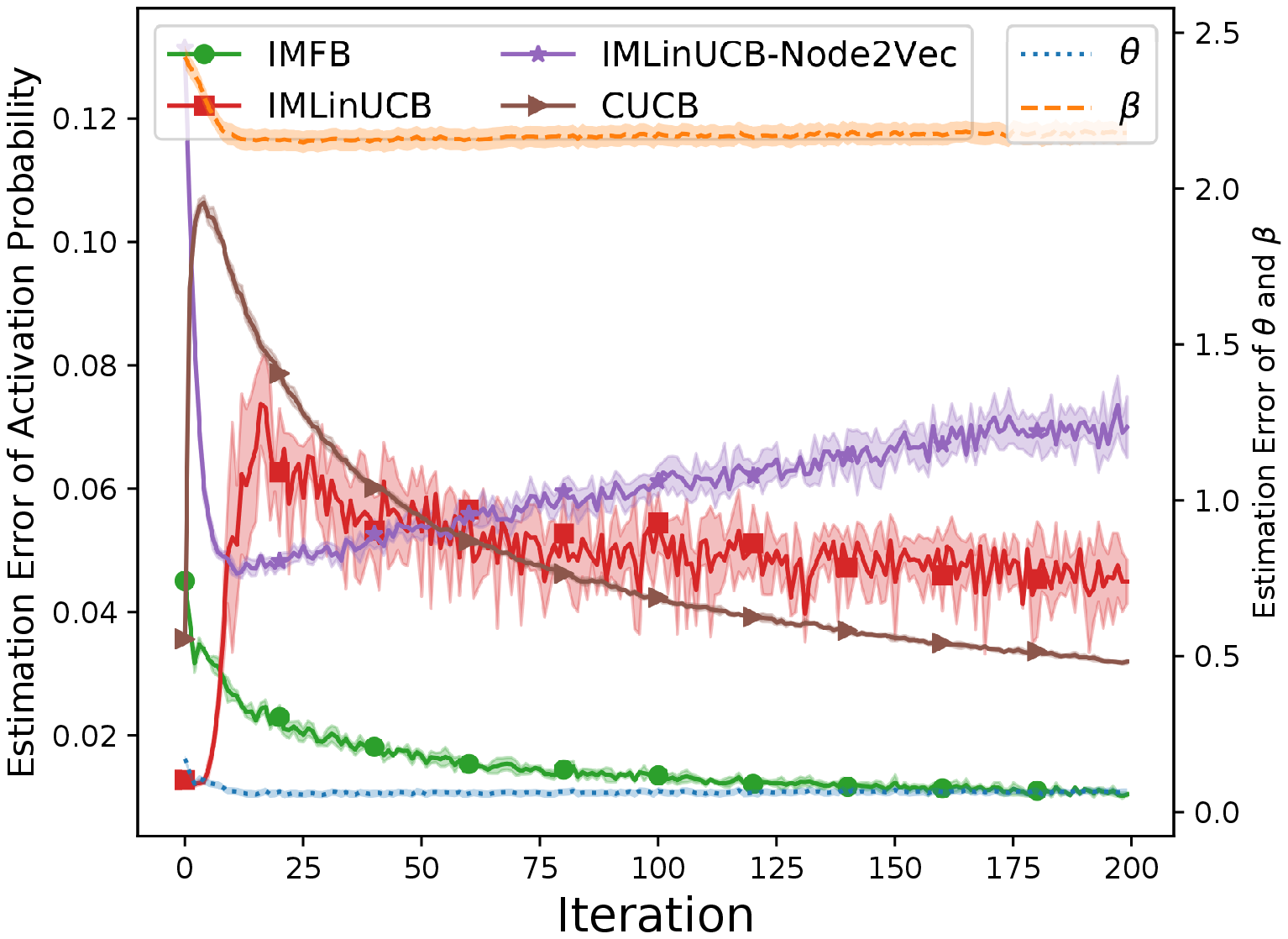} \\
(a) NetHEPT  &  (b) Flickr  \\
\end{tabular}
\vspace{-3mm}
\caption{Influence model estimation quality.}  \label{fig:loss}
\vspace{-4mm}
\end{figure}

From the results on NetHEPT in Figure \ref{fig:loss} (a), we can observe that, in addition to IMFB, the estimation error of activation probabilities in IMLinUCB also converge very fast but still worse than IMFB.  IMLinUCB-Node2Vec has a much worse activation probability estimation than all the other baselines. The unsatisfactory activation probability estimation in CUCB again verified our motivation of joint edge level parameter estimation. We can see that at around round 85, there is a sudden jump of its activation probability estimation quality. This is because at that round CUCB starts to explore edges that were rarely observed before. For those  edges, CUCB has to learn the activation probability from scratch, which leads to a sudden drop and a large variance in its estimation. This observation directly verify phenomenon shown in the illustration example in Figure \ref{fig:illu}, where we showed CUCB can only improve the estimation of activation probabilities on edges that are directly observed.

On the Flickr dataset in Figure \ref{fig:loss} (b), we observe similar estimation error pattern in CUCB at round interaction 15. Since this happens much earlier than that on the NetHEPT dataset, the performance of CUCB is not greatly affected after all. And on this dataset, IMFB converges the fastest and converges to an error that is much smaller than any other baselines'. Results on both these two dataset empirically confirmed the reduced sample complexity of IMFB.

In addition, we also visualize the average L2 difference between the estimated factors in IMFB and the ground-truth factors in the environment, i.e., $\lVert \hat \btheta_{v,t} - \btheta^*_v \rVert_2$ and $\lVert \hat \Beta{v,t} - \Beta^*_v \rVert_2$, on \emph{observed} nodes $v$. The results on both datasets are shown in Figure \ref{fig:loss} (a) \& (b) with dotted lines using the left side of the y-axis. We can find that within 25 iterations, IMFB's estimation error in both $\btheta$ and $\Beta$ reduces. As a result, its estimated activation probability $\hat p_{e,t}$ quickly converges to ground-truth probability, which corresponds to the convergence of the real time reward. The good estimation quality of activation probability in IMFB coincides our analysis in Lemma \ref{lemma:activation_prob_CB} regarding the confidence bound of activation probability estimation.

\subsection{Robustness in Influence Model Learning}
\subsubsection{Dimensionality Misspecification}
In our proposed solution, dimensionality of the latent influence factors and susceptibility factors need to be specified as the input. On the tested datasets, the dimensionality is set to 20 when generating the ground-truth activation probabilities. However, such information would be unknown to a learner in practice. Thus it is important to study the robustness of our proposed IMFB algorithm when dimensionality of latent factor is misspecified. Here we use IMFB-$d'$ to denote IMFB with latent dimension size $d'$. We vary $d'$ to study IMFB's robustness to dimensionality misspecification. Average reward on NetHEPT and Flickr dataset are reported in Table \ref{tab:robustness_dimension}. We also include results from all the other baselines for comparison. From this result we can clearly find that: firstly IMFB is very robust to dimensionality misspecification, specifically when $d'$ is set within a reasonable range. For example when $d'$ is in the range of 5 to 60, given the ground-truth dimension size is 20, the variance in the performance of IMFB is very small. Even when $d'$ is set to 1, IMFB can still maintain competitive performance. Secondly, by comparing IMFB with all  other baselines, we can see that even when the dimensionality is misspecified, IMFB can still significantly outperform all the baselines. This suggests specifying the correct influence structure is more important for online influence maximization.

\subsubsection{Different Environment Settings}
To examine our algorithm's robustness to different environment setups, we performed evaluation on the Flickr dataset by varying several important parameters of the environment, and report the results in Table ~\ref{tab:robust}.

\noindent\textbf{$\bullet$ Number of seed nodes $K$.} We vary the number of seed nodes $K$ from 50 to 1,000. We observe that for different settings of $K$, IMFB consistently performs better than other baselines. 

\noindent\textbf{$\bullet$ Noise level $\eta$.} Previous works in online influence maximization assumes the activation of edge $e$ is controlled by a fixed ground-truth activation probability $p_e$. 
In order to get a sense of our algorithm's robustness to additional noise on the edge activation, at each round of interaction we add an additional noise term $\eta$ on the ground-truth activation probabilities and then project the resulting noisy activation probabilities to $[0,1]$. From Table ~\ref{tab:robust} we observe that IMFB is robust to different noise levels and outperforms other baselines even when the noise scale is very large. 

\noindent\textbf{$\bullet$ Activation probability scale coefficient $c$.} To simulate some practical scenarios where the graph is generally more active or inactive, we re-scale the activation probability $p_e$ in our default setting by a constant coefficient $c$, i.e., activate edge $e$ with probability $c \cdot p_e$ , where $c$ is varied from $0.5$ to $2.0$. 
According to the results in Table ~\ref{tab:robust}, our proposed IMFB is robust to the choice of $c$ and generally improves over other baselines. We also observe that when $c=2.0$ the difference between all algorithms becomes much smaller. This is expected: when $c$ is large, the network is very active in every node, which makes the number of observations much larger which reduces necessity of observation propagation.  



\section{Conclusions \& Future Work}
 Motivated by the network assortative mixing property, we decompose the edge activation probability into giving node's influence and receiving node's susceptibility, and therefore perform node-level influence model estimation for better model complexity and generalization in online influence maximization. 
We provide rigorous theoretical analysis of IMFB, and show considerable regret reduction comparing with several state-of-the-art online maximization solutions. We also empirically evaluate on two real-world network datasets and confirm the effectiveness and robustness of IMFB. 

Our factorization based solution provides a flexible framework to introduce more desired properties, e.g., homophily and simultaneity, to influence modeling. Currently, we have assumed the availability of edge-level feedback from influence maximization. But in practice node-level feedback is more prevalent, we can incorporate more sophisticated inference methods based on our factor models to learn from such feedback.   

\section{Acknowledgements}
We thank the anonymous reviewers for their insightful comments.
This work was supported in part by National Science Foundation
Grant IIS-1553568 and IIS-1618948 and Bloomberg Data Science
PhD Fellowship. Dr. Wei Chen is partially supported by the National Natural Science Foundation of China (Grant NO.61433014).
\bibliographystyle{ACM-Reference-Format}
\bibliography{IM} 


\begin{thebibliography}{00}


\ifx \showCODEN    \undefined \def \showCODEN     #1{\unskip}     \fi
\ifx \showDOI      \undefined \def \showDOI       #1{#1}\fi
\ifx \showISBNx    \undefined \def \showISBNx     #1{\unskip}     \fi
\ifx \showISBNxiii \undefined \def \showISBNxiii  #1{\unskip}     \fi
\ifx \showISSN     \undefined \def \showISSN      #1{\unskip}     \fi
\ifx \showLCCN     \undefined \def \showLCCN      #1{\unskip}     \fi
\ifx \shownote     \undefined \def \shownote      #1{#1}          \fi
\ifx \showarticletitle \undefined \def \showarticletitle #1{#1}   \fi
\ifx \showURL      \undefined \def \showURL       {\relax}        \fi
\providecommand\bibfield[2]{#2}
\providecommand\bibinfo[2]{#2}
\providecommand\natexlab[1]{#1}
\providecommand\showeprint[2][]{arXiv:#2}

\bibitem[\protect\citeauthoryear{Aral and Dhillon}{Aral and Dhillon}{2018}]%
        {Aral2018}
\bibfield{author}{\bibinfo{person}{Sinan Aral} {and}
  \bibinfo{person}{Paramveer~S. Dhillon}.} \bibinfo{year}{2018}\natexlab{}.
\newblock \showarticletitle{Social influence maximization under empirical
  influence models}.
\newblock \bibinfo{journal}{{\em Nature Human Behaviour\/}}
  \bibinfo{volume}{6}, \bibinfo{number}{2} (\bibinfo{year}{2018}),
  \bibinfo{pages}{375--382}.
\newblock


\bibitem[\protect\citeauthoryear{Aral and Walker}{Aral and Walker}{2012}]%
        {Aral337}
\bibfield{author}{\bibinfo{person}{Sinan Aral} {and} \bibinfo{person}{Dylan
  Walker}.} \bibinfo{year}{2012}\natexlab{}.
\newblock \showarticletitle{Identifying Influential and Susceptible Members of
  Social Networks}.
\newblock \bibinfo{journal}{{\em Science\/}} \bibinfo{volume}{337},
  \bibinfo{number}{6092} (\bibinfo{year}{2012}), \bibinfo{pages}{337--341}.
\newblock


\bibitem[\protect\citeauthoryear{Auer}{Auer}{2002}]%
        {Auer02}
\bibfield{author}{\bibinfo{person}{Peter Auer}.}
  \bibinfo{year}{2002}\natexlab{}.
\newblock \showarticletitle{Using Confidence Bounds for
  Exploitation-Exploration Trade-offs}.
\newblock \bibinfo{journal}{{\em Journal of Machine Learning Research\/}}
  \bibinfo{volume}{3} (\bibinfo{year}{2002}), \bibinfo{pages}{397--422}.
\newblock


\bibitem[\protect\citeauthoryear{Bourigault, Lamprier, and
  Gallinari}{Bourigault et~al\mbox{.}}{2016}]%
        {bourigault2016representation}
\bibfield{author}{\bibinfo{person}{Simon Bourigault}, \bibinfo{person}{Sylvain
  Lamprier}, {and} \bibinfo{person}{Patrick Gallinari}.}
  \bibinfo{year}{2016}\natexlab{}.
\newblock \showarticletitle{Representation learning for information diffusion
  through social networks: an embedded cascade model}. In
  \bibinfo{booktitle}{{\em Proceedings of the 9th ACM WSDM}}. ACM,
  \bibinfo{pages}{573--582}.
\newblock


\bibitem[\protect\citeauthoryear{Carpentier and Valko}{Carpentier and
  Valko}{2016}]%
        {carpentier2016revealing}
\bibfield{author}{\bibinfo{person}{Alexandra Carpentier} {and}
  \bibinfo{person}{Michal Valko}.} \bibinfo{year}{2016}\natexlab{}.
\newblock \showarticletitle{Revealing graph bandits for maximizing local
  influence}. In \bibinfo{booktitle}{{\em AISTATS}}. \bibinfo{pages}{10--18}.
\newblock


\bibitem[\protect\citeauthoryear{Centola and Macy}{Centola and Macy}{2007}]%
        {centola2007complex}
\bibfield{author}{\bibinfo{person}{Damon Centola} {and}
  \bibinfo{person}{Michael Macy}.} \bibinfo{year}{2007}\natexlab{}.
\newblock \showarticletitle{Complex contagions and the weakness of long ties}.
\newblock \bibinfo{journal}{{\em American journal of Sociology\/}}
  \bibinfo{volume}{113}, \bibinfo{number}{3} (\bibinfo{year}{2007}),
  \bibinfo{pages}{702--734}.
\newblock


\bibitem[\protect\citeauthoryear{Chen, Wang, and Wang}{Chen
  et~al\mbox{.}}{2010}]%
        {chen2010scalable}
\bibfield{author}{\bibinfo{person}{Wei Chen}, \bibinfo{person}{Chi Wang}, {and}
  \bibinfo{person}{Yajun Wang}.} \bibinfo{year}{2010}\natexlab{}.
\newblock \showarticletitle{Scalable influence maximization for prevalent viral
  marketing in large-scale social networks}. In \bibinfo{booktitle}{{\em
  Proceedings of the 16th ACM SIGKDD}}. ACM, \bibinfo{pages}{1029--1038}.
\newblock


\bibitem[\protect\citeauthoryear{Chen, Wang, and Yang}{Chen
  et~al\mbox{.}}{2009}]%
        {chen2009efficient}
\bibfield{author}{\bibinfo{person}{Wei Chen}, \bibinfo{person}{Yajun Wang},
  {and} \bibinfo{person}{Siyu Yang}.} \bibinfo{year}{2009}\natexlab{}.
\newblock \showarticletitle{Efficient influence maximization in social
  networks}. In \bibinfo{booktitle}{{\em Proceedings of the 15th ACM SIGKDD}}.
  ACM, \bibinfo{pages}{199--208}.
\newblock


\bibitem[\protect\citeauthoryear{Chen, Wang, and Yuan}{Chen
  et~al\mbox{.}}{2013}]%
        {chen2013combinatorial}
\bibfield{author}{\bibinfo{person}{Wei Chen}, \bibinfo{person}{Yajun Wang},
  {and} \bibinfo{person}{Yang Yuan}.} \bibinfo{year}{2013}\natexlab{}.
\newblock \showarticletitle{Combinatorial multi-armed bandit: General framework
  and applications}. In \bibinfo{booktitle}{{\em ICML}}.
  \bibinfo{pages}{151--159}.
\newblock


\bibitem[\protect\citeauthoryear{Chen, Wang, Yuan, and Wang}{Chen
  et~al\mbox{.}}{2016}]%
        {Chen:2016:CMB:2946645.2946695}
\bibfield{author}{\bibinfo{person}{Wei Chen}, \bibinfo{person}{Yajun Wang},
  \bibinfo{person}{Yang Yuan}, {and} \bibinfo{person}{Qinshi Wang}.}
  \bibinfo{year}{2016}\natexlab{}.
\newblock \showarticletitle{Combinatorial Multi-armed Bandit and Its Extension
  to Probabilistically Triggered Arms}.
\newblock \bibinfo{journal}{{\em J. Mach. Learn. Res.\/}} \bibinfo{volume}{17},
  \bibinfo{number}{1} (\bibinfo{date}{Jan.} \bibinfo{year}{2016}),
  \bibinfo{pages}{1746--1778}.
\newblock
\showISSN{1532-4435}


\bibitem[\protect\citeauthoryear{El-Sayed, Scarborough, Seemann, and
  Galea}{El-Sayed et~al\mbox{.}}{2012}]%
        {el2012social}
\bibfield{author}{\bibinfo{person}{Abdulrahman~M El-Sayed},
  \bibinfo{person}{Peter Scarborough}, \bibinfo{person}{Lars Seemann}, {and}
  \bibinfo{person}{Sandro Galea}.} \bibinfo{year}{2012}\natexlab{}.
\newblock \showarticletitle{Social network analysis and agent-based modeling in
  social epidemiology}.
\newblock \bibinfo{journal}{{\em Epidemiologic Perspectives \& Innovations\/}}
  \bibinfo{volume}{9}, \bibinfo{number}{1} (\bibinfo{year}{2012}),
  \bibinfo{pages}{1}.
\newblock


\bibitem[\protect\citeauthoryear{Golovin and Krause}{Golovin and
  Krause}{2011}]%
        {golovin2011adaptive}
\bibfield{author}{\bibinfo{person}{Daniel Golovin} {and}
  \bibinfo{person}{Andreas Krause}.} \bibinfo{year}{2011}\natexlab{}.
\newblock \showarticletitle{Adaptive submodularity: Theory and applications in
  active learning and stochastic optimization}.
\newblock \bibinfo{journal}{{\em Journal of Artificial Intelligence
  Research\/}}  \bibinfo{volume}{42} (\bibinfo{year}{2011}),
  \bibinfo{pages}{427--486}.
\newblock


\bibitem[\protect\citeauthoryear{Goyal, Bonchi, and Lakshmanan}{Goyal
  et~al\mbox{.}}{2010}]%
        {goyal2010learning}
\bibfield{author}{\bibinfo{person}{Amit Goyal}, \bibinfo{person}{Francesco
  Bonchi}, {and} \bibinfo{person}{Laks~VS Lakshmanan}.}
  \bibinfo{year}{2010}\natexlab{}.
\newblock \showarticletitle{Learning influence probabilities in social
  networks}. In \bibinfo{booktitle}{{\em Proceedings of the third ACM WSDM}}.
  ACM, \bibinfo{pages}{241--250}.
\newblock


\bibitem[\protect\citeauthoryear{Grover and Leskovec}{Grover and
  Leskovec}{2016}]%
        {grover2016node2vec}
\bibfield{author}{\bibinfo{person}{Aditya Grover} {and} \bibinfo{person}{Jure
  Leskovec}.} \bibinfo{year}{2016}\natexlab{}.
\newblock \showarticletitle{node2vec: Scalable feature learning for networks}.
  In \bibinfo{booktitle}{{\em Proceedings of the 22nd ACM SIGKDD}}. ACM,
  \bibinfo{pages}{855--864}.
\newblock


\bibitem[\protect\citeauthoryear{Kalimeris, Singer, Subbian, and
  Weinsberg}{Kalimeris et~al\mbox{.}}{2018}]%
        {pmlr-v80-kalimeris18a}
\bibfield{author}{\bibinfo{person}{Dimitris Kalimeris}, \bibinfo{person}{Yaron
  Singer}, \bibinfo{person}{Karthik Subbian}, {and} \bibinfo{person}{Udi
  Weinsberg}.} \bibinfo{year}{2018}\natexlab{}.
\newblock \showarticletitle{Learning Diffusion using Hyperparameters}. In
  \bibinfo{booktitle}{{\em Proceedings of the 35th ICML}},
  Vol.~\bibinfo{volume}{80}. \bibinfo{publisher}{PMLR},
  \bibinfo{address}{Stockholmsmässan, Stockholm Sweden},
  \bibinfo{pages}{2420--2428}.
\newblock


\bibitem[\protect\citeauthoryear{Kempe, Kleinberg, and Tardos}{Kempe
  et~al\mbox{.}}{2003}]%
        {kempe2003maximizing}
\bibfield{author}{\bibinfo{person}{David Kempe}, \bibinfo{person}{Jon
  Kleinberg}, {and} \bibinfo{person}{{\'E}va Tardos}.}
  \bibinfo{year}{2003}\natexlab{}.
\newblock \showarticletitle{Maximizing the spread of influence through a social
  network}. In \bibinfo{booktitle}{{\em Proceedings of the ninth ACM SIGKDD}}.
  ACM, \bibinfo{pages}{137--146}.
\newblock


\bibitem[\protect\citeauthoryear{Kitsak, Gallos, Havlin, Liljeros, Muchnik,
  Stanley, and Makse}{Kitsak et~al\mbox{.}}{2010}]%
        {kitsak2010identification}
\bibfield{author}{\bibinfo{person}{Maksim Kitsak}, \bibinfo{person}{Lazaros~K
  Gallos}, \bibinfo{person}{Shlomo Havlin}, \bibinfo{person}{Fredrik Liljeros},
  \bibinfo{person}{Lev Muchnik}, \bibinfo{person}{H~Eugene Stanley}, {and}
  \bibinfo{person}{Hern{\'a}n~A Makse}.} \bibinfo{year}{2010}\natexlab{}.
\newblock \showarticletitle{Identification of influential spreaders in complex
  networks}.
\newblock \bibinfo{journal}{{\em Nature physics\/}} \bibinfo{volume}{6},
  \bibinfo{number}{11} (\bibinfo{year}{2010}), \bibinfo{pages}{888}.
\newblock


\bibitem[\protect\citeauthoryear{Lei, Maniu, Mo, Cheng, and Senellart}{Lei
  et~al\mbox{.}}{2015}]%
        {lei2015online}
\bibfield{author}{\bibinfo{person}{Siyu Lei}, \bibinfo{person}{Silviu Maniu},
  \bibinfo{person}{Luyi Mo}, \bibinfo{person}{Reynold Cheng}, {and}
  \bibinfo{person}{Pierre Senellart}.} \bibinfo{year}{2015}\natexlab{}.
\newblock \showarticletitle{Online influence maximization}. In
  \bibinfo{booktitle}{{\em Proceedings of the 21th ACM SIGKDD}}. ACM,
  \bibinfo{pages}{645--654}.
\newblock


\bibitem[\protect\citeauthoryear{McPherson, Smith-Lovin, and Cook}{McPherson
  et~al\mbox{.}}{2001}]%
        {mcpherson2001birds}
\bibfield{author}{\bibinfo{person}{Miller McPherson}, \bibinfo{person}{Lynn
  Smith-Lovin}, {and} \bibinfo{person}{James~M Cook}.}
  \bibinfo{year}{2001}\natexlab{}.
\newblock \showarticletitle{Birds of a feather: Homophily in social networks}.
\newblock \bibinfo{journal}{{\em Annual review of sociology\/}}
  \bibinfo{volume}{27}, \bibinfo{number}{1} (\bibinfo{year}{2001}),
  \bibinfo{pages}{415--444}.
\newblock


\bibitem[\protect\citeauthoryear{Mislove, Marcon, Gummadi, Druschel, and
  Bhattacharjee}{Mislove et~al\mbox{.}}{2007}]%
        {mislove2007measurement}
\bibfield{author}{\bibinfo{person}{Alan Mislove}, \bibinfo{person}{Massimiliano
  Marcon}, \bibinfo{person}{Krishna~P Gummadi}, \bibinfo{person}{Peter
  Druschel}, {and} \bibinfo{person}{Bobby Bhattacharjee}.}
  \bibinfo{year}{2007}\natexlab{}.
\newblock \showarticletitle{Measurement and analysis of online social
  networks}. In \bibinfo{booktitle}{{\em Proceedings of the 7th ACM SIGCOMM}}.
  ACM, \bibinfo{pages}{29--42}.
\newblock


\bibitem[\protect\citeauthoryear{Morgan and Winship}{Morgan and
  Winship}{2014}]%
        {morgan2014counterfactuals}
\bibfield{author}{\bibinfo{person}{Stephen~L Morgan} {and}
  \bibinfo{person}{Christopher Winship}.} \bibinfo{year}{2014}\natexlab{}.
\newblock \bibinfo{booktitle}{{\em Counterfactuals and causal inference}}.
\newblock \bibinfo{publisher}{Cambridge University Press}.
\newblock


\bibitem[\protect\citeauthoryear{Nemhauser, Wolsey, and Fisher}{Nemhauser
  et~al\mbox{.}}{1978}]%
        {nemhauser1978analysis}
\bibfield{author}{\bibinfo{person}{George~L Nemhauser},
  \bibinfo{person}{Laurence~A Wolsey}, {and} \bibinfo{person}{Marshall~L
  Fisher}.} \bibinfo{year}{1978}\natexlab{}.
\newblock \showarticletitle{An analysis of approximations for maximizing
  submodular set functions}.
\newblock \bibinfo{journal}{{\em Mathematical Programming\/}}
  \bibinfo{volume}{14}, \bibinfo{number}{1} (\bibinfo{year}{1978}),
  \bibinfo{pages}{265--294}.
\newblock


\bibitem[\protect\citeauthoryear{Netrapalli and Sanghavi}{Netrapalli and
  Sanghavi}{2012}]%
        {Netrapalli:2012:LGE:2318857.2254783}
\bibfield{author}{\bibinfo{person}{Praneeth Netrapalli} {and}
  \bibinfo{person}{Sujay Sanghavi}.} \bibinfo{year}{2012}\natexlab{}.
\newblock \showarticletitle{Learning the Graph of Epidemic Cascades}.
\newblock \bibinfo{journal}{{\em SIGMETRICS Perform. Eval. Rev.\/}}
  \bibinfo{volume}{40}, \bibinfo{number}{1} (\bibinfo{date}{June}
  \bibinfo{year}{2012}), \bibinfo{pages}{211--222}.
\newblock
\showISSN{0163-5999}
\showDOI{%
\url{https://doi.org/10.1145/2318857.2254783}}


\bibitem[\protect\citeauthoryear{Newman}{Newman}{2002}]%
        {newman2002assortative}
\bibfield{author}{\bibinfo{person}{Mark~EJ Newman}.}
  \bibinfo{year}{2002}\natexlab{}.
\newblock \showarticletitle{Assortative mixing in networks}.
\newblock \bibinfo{journal}{{\em Physical review letters\/}}
  \bibinfo{volume}{89}, \bibinfo{number}{20} (\bibinfo{year}{2002}),
  \bibinfo{pages}{208701}.
\newblock


\bibitem[\protect\citeauthoryear{Nguyen, Thai, and Dinh}{Nguyen
  et~al\mbox{.}}{2016}]%
        {Nguyen:2016:SOS:2882903.2915207}
\bibfield{author}{\bibinfo{person}{Hung~T. Nguyen}, \bibinfo{person}{My~T.
  Thai}, {and} \bibinfo{person}{Thang~N. Dinh}.}
  \bibinfo{year}{2016}\natexlab{}.
\newblock \showarticletitle{Stop-and-Stare: Optimal Sampling Algorithms for
  Viral Marketing in Billion-scale Networks}. In \bibinfo{booktitle}{{\em
  Proceedings of the 2016 ACM SIGMOD}}. \bibinfo{publisher}{ACM},
  \bibinfo{address}{New York, NY, USA}, \bibinfo{pages}{695--710}.
\newblock
\showISBNx{978-1-4503-3531-7}
\showDOI{%
\url{https://doi.org/10.1145/2882903.2915207}}


\bibitem[\protect\citeauthoryear{Olkhovskaya, Neu, and Lugosi}{Olkhovskaya
  et~al\mbox{.}}{2018}]%
        {olkhovskaya2018online}
\bibfield{author}{\bibinfo{person}{Julia Olkhovskaya}, \bibinfo{person}{Gergely
  Neu}, {and} \bibinfo{person}{Gábor Lugosi}.}
  \bibinfo{year}{2018}\natexlab{}.
\newblock \bibinfo{title}{Online Influence Maximization with Local
  Observations}.
\newblock   (\bibinfo{year}{2018}).
\newblock
\showeprint[arxiv]{cs.LG/1805.11022}


\bibitem[\protect\citeauthoryear{Saito, Nakano, and Kimura}{Saito
  et~al\mbox{.}}{2008}]%
        {Saito:2008:PID:1430307.1430318}
\bibfield{author}{\bibinfo{person}{Kazumi Saito}, \bibinfo{person}{Ryohei
  Nakano}, {and} \bibinfo{person}{Masahiro Kimura}.}
  \bibinfo{year}{2008}\natexlab{}.
\newblock \showarticletitle{Prediction of Information Diffusion Probabilities
  for Independent Cascade Model}. In \bibinfo{booktitle}{{\em Proceedings of
  the 12th KES}}. \bibinfo{pages}{67--75}.
\newblock
\showISBNx{978-3-540-85566-8}


\bibitem[\protect\citeauthoryear{Sutton, Barto, et~al\mbox{.}}{Sutton
  et~al\mbox{.}}{1998}]%
        {sutton1998introduction}
\bibfield{author}{\bibinfo{person}{Richard~S Sutton}, \bibinfo{person}{Andrew~G
  Barto}, {et~al\mbox{.}}} \bibinfo{year}{1998}\natexlab{}.
\newblock \bibinfo{booktitle}{{\em Introduction to reinforcement learning}}.
  Vol.~\bibinfo{volume}{135}.
\newblock \bibinfo{publisher}{MIT press Cambridge}.
\newblock


\bibitem[\protect\citeauthoryear{Tang, Shi, and Xiao}{Tang
  et~al\mbox{.}}{2015}]%
        {Tang:2015:IMN:2723372.2723734}
\bibfield{author}{\bibinfo{person}{Youze Tang}, \bibinfo{person}{Yanchen Shi},
  {and} \bibinfo{person}{Xiaokui Xiao}.} \bibinfo{year}{2015}\natexlab{}.
\newblock \showarticletitle{Influence Maximization in Near-Linear Time: A
  Martingale Approach}. In \bibinfo{booktitle}{{\em Proceedings of the 2015 ACM
  SIGMOD}}. \bibinfo{publisher}{ACM}, \bibinfo{address}{New York, NY, USA},
  \bibinfo{pages}{1539--1554}.
\newblock
\showISBNx{978-1-4503-2758-9}


\bibitem[\protect\citeauthoryear{Tang, Xiao, and Shi}{Tang
  et~al\mbox{.}}{2014}]%
        {Tang:2014:IMN:2588555.2593670}
\bibfield{author}{\bibinfo{person}{Youze Tang}, \bibinfo{person}{Xiaokui Xiao},
  {and} \bibinfo{person}{Yanchen Shi}.} \bibinfo{year}{2014}\natexlab{}.
\newblock \showarticletitle{Influence Maximization: Near-optimal Time
  Complexity Meets Practical Efficiency}. In \bibinfo{booktitle}{{\em
  Proceedings of the 2014 ACM SIGMOD}}. \bibinfo{publisher}{ACM},
  \bibinfo{address}{New York, NY, USA}, \bibinfo{pages}{75--86}.
\newblock
\showISBNx{978-1-4503-2376-5}


\bibitem[\protect\citeauthoryear{Uschmajew}{Uschmajew}{2012}]%
        {LocalConvergenceALS}
\bibfield{author}{\bibinfo{person}{Andr{\'e} Uschmajew}.}
  \bibinfo{year}{2012}\natexlab{}.
\newblock \showarticletitle{Local convergence of the alternating least squares
  algorithm for canonical tensor approximation}.
\newblock \bibinfo{journal}{{\it SIAM J. Matrix Anal. Appl.}}
  \bibinfo{volume}{33}, \bibinfo{number}{2} (\bibinfo{year}{2012}),
  \bibinfo{pages}{639--652}.
\newblock


\bibitem[\protect\citeauthoryear{Vaswani, Kveton, Wen, Ghavamzadeh, Lakshmanan,
  and Schmidt}{Vaswani et~al\mbox{.}}{2017}]%
        {vaswani2017model}
\bibfield{author}{\bibinfo{person}{Sharan Vaswani}, \bibinfo{person}{Branislav
  Kveton}, \bibinfo{person}{Zheng Wen}, \bibinfo{person}{Mohammad Ghavamzadeh},
  \bibinfo{person}{Laks~VS Lakshmanan}, {and} \bibinfo{person}{Mark Schmidt}.}
  \bibinfo{year}{2017}\natexlab{}.
\newblock \showarticletitle{Model-independent online learning for influence
  maximization}. In \bibinfo{booktitle}{{\em ICML}}.
  \bibinfo{pages}{3530--3539}.
\newblock


\bibitem[\protect\citeauthoryear{Vaswani, Lakshmanan, Schmidt,
  et~al\mbox{.}}{Vaswani et~al\mbox{.}}{2015}]%
        {vaswani2015influence}
\bibfield{author}{\bibinfo{person}{Sharan Vaswani}, \bibinfo{person}{Laks
  Lakshmanan}, \bibinfo{person}{Mark Schmidt}, {et~al\mbox{.}}}
  \bibinfo{year}{2015}\natexlab{}.
\newblock \showarticletitle{Influence Maximization with Bandits}.
\newblock \bibinfo{journal}{{\em arXiv preprint arXiv:1503.00024\/}}
  (\bibinfo{year}{2015}).
\newblock


\bibitem[\protect\citeauthoryear{Wang, Wu, and Wang}{Wang
  et~al\mbox{.}}{2016}]%
        {wang2016learning}
\bibfield{author}{\bibinfo{person}{Huazheng Wang}, \bibinfo{person}{Qingyun
  Wu}, {and} \bibinfo{person}{Hongning Wang}.} \bibinfo{year}{2016}\natexlab{}.
\newblock \showarticletitle{Learning hidden features for contextual bandits}.
  In \bibinfo{booktitle}{{\em Proceedings of the 25th ACM CIKM}}. ACM,
  \bibinfo{pages}{1633--1642}.
\newblock


\bibitem[\protect\citeauthoryear{Wang, Wu, and Wang}{Wang
  et~al\mbox{.}}{2017}]%
        {wang2017factorization}
\bibfield{author}{\bibinfo{person}{Huazheng Wang}, \bibinfo{person}{Qingyun
  Wu}, {and} \bibinfo{person}{Hongning Wang}.} \bibinfo{year}{2017}\natexlab{}.
\newblock \showarticletitle{Factorization bandits for interactive
  recommendation}. In \bibinfo{booktitle}{{\em Thirty-First AAAI Conference on
  Artificial Intelligence}}.
\newblock


\bibitem[\protect\citeauthoryear{Wang and Chen}{Wang and Chen}{2017}]%
        {wang2017improving}
\bibfield{author}{\bibinfo{person}{Qinshi Wang} {and} \bibinfo{person}{Wei
  Chen}.} \bibinfo{year}{2017}\natexlab{}.
\newblock \showarticletitle{Improving regret bounds for combinatorial
  semi-bandits with probabilistically triggered arms and its applications}. In
  \bibinfo{booktitle}{{\em NIPS}}. \bibinfo{pages}{1161--1171}.
\newblock


\bibitem[\protect\citeauthoryear{Watts and Dodds}{Watts and Dodds}{2007}]%
        {watts2007influentials}
\bibfield{author}{\bibinfo{person}{Duncan~J Watts} {and}
  \bibinfo{person}{Peter~Sheridan Dodds}.} \bibinfo{year}{2007}\natexlab{}.
\newblock \showarticletitle{Influentials, networks, and public opinion
  formation}.
\newblock \bibinfo{journal}{{\em Journal of consumer research\/}}
  \bibinfo{volume}{34}, \bibinfo{number}{4} (\bibinfo{year}{2007}),
  \bibinfo{pages}{441--458}.
\newblock


\bibitem[\protect\citeauthoryear{Wen, Kveton, Valko, and Vaswani}{Wen
  et~al\mbox{.}}{2017}]%
        {wen2017online}
\bibfield{author}{\bibinfo{person}{Zheng Wen}, \bibinfo{person}{Branislav
  Kveton}, \bibinfo{person}{Michal Valko}, {and} \bibinfo{person}{Sharan
  Vaswani}.} \bibinfo{year}{2017}\natexlab{}.
\newblock \showarticletitle{Online influence maximization under independent
  cascade model with semi-bandit feedback}. In \bibinfo{booktitle}{{\em NIPS}}.
  \bibinfo{pages}{3022--3032}.
\newblock


\end{thebibliography}

\clearpage
\section{Supplement}
\subsection{Implementation of IMFB}
The implementation of our proposed solution IMFB is publicly available at \url{https://github.com/Matrix-Factorization-Bandit/IMFB-KDD2019}. 
\subsection{Proof of Theorem \ref{theorem:main_regretbound} and related lemmas}

To carry out our analysis,
we make the following mild conditions on the
expected reward $f_{\bP}(\cS)$. Note that these two conditions are the same as stated in \cite{wang2017improving}. 
\begin{itemize}
    \item \textbf{Monotonicity}. The expected reward of playing any super arm $S \in \bS$ is monotonically non-decreasing with respect to the expectation vector, i.e., if for all $i \in [m]$ $p_i \leq p_i'$ , we have $f_{\bP}(S) \leq f_{\bP'}(S)$ for all $S \in \bS$, in which $\cS$ is the set of all candidate super arms.
\end{itemize}

\begin{itemize}
    \item \textbf{1-Norm Bounded Smoothness}. We say that a CMAB-T (Combinatorial Multi-Armed Bandit with probabilistically Triggered arms) satisfy 1-norm bounded smoothness, if there exists a bounded smoothness constant $B \in \mathbb{R}^{+}$ such that for any two distributions with expectation vectors $\bP$ and $\bP'$, and any action $\cS$, we have $|f_{\bP}(\cS) - f_{\bP'}(\cS)| \leq B \sum_{i \in \tilde{\bS}} |p_i - p_i'|$, where $\tilde{\cS}$ is the set of arms that are triggered by $\cS$. 
\end{itemize}

\begin{proof}[Proof of Theorem \ref{theorem:main_regretbound}]
According to the definition of scaled regret, we have
\begin{align}
    \bbE[R_t^{\alpha \gamma}] & =  f_{\bP^*}(\cS^{opt}) - \bbE[ \frac{1}{\alpha \gamma} f_{\bP^*}(\cS_t)] \\  \nonumber
    & \leq \frac{1}{\alpha \gamma } \bbE[f_{\bP^*}(\cS^*) - f_{\bP^*}(\cS_t)]  \nonumber
\end{align}
where the expectation is over the possible randomness of $\cS_t$ since the oracle might be a randomized algorithm. Notice that the randomness coming from the edge activation is already taken care of in the definition of $r$ function. For any $t\leq T$, we define event  $\xi_{t}$ as,
\begin{align}
    \xi_{t-1}  = \{ & |p_e^* - \hat p_{e,t}| \leq  \alpha_{v_{g_e}}^{\beta} \lVert \hat \Beta_{g_e,i-1}\rVert_{\bA_{g_e,i-1}^{-1}} \\  \nonumber
    & + \alpha_{r_e}^{\theta} \lVert \hat \btheta_{v_{r_e},i-1}\rVert_{\bC_{r_e, i-1}^{-1}} + 2q^{2i}, \forall e \in \cE, \forall i \leq t\}
\end{align}
and $\bar \xi_{t-1}$ as the complement event of $\xi_{t-1}$. Hence we have,
\begin{align} \label{eq:regret_decomposition}
    \bbE[R_t^{\alpha \gamma}] & \leq \frac{ \bbP(\xi_{t-1})}{\alpha \gamma} \bbE[f_{\bP^*}(\cS^*) - f_{\bP^*}(\cS_t)|\xi_{t-1}] + \bbP(\bar \xi_{t-1}) (L-K) \\ \nonumber
    & \leq \frac{ \bbP(\xi_{t-1})}{\alpha \gamma} \bbE[f_{\bar \bP_t}(\cS^*) - f_{\bP^*}(\cS_t)|\xi_{t-1}] + \bbP(\bar \xi_{t-1}) (L-K) \\ \nonumber
    &  \leq \frac{ \bbP(\xi_{t-1})}{\alpha \gamma} \bbE[f_{\bar \bP_t}(\cS_t) - f_{\bP^*}(\cS_t)|\xi_{t-1}] + \bbP(\bar \xi_{t-1}) (L-K) \nonumber
\end{align}
in which the second inequality is based on the monotonicity of the expected reward function (since we proved that with high probability each element in $\bar \bP_t$ is an upper bound of the coressponding element in $\bP^*$. ), and the third inequality is based on the oracle's seed node selection strategy.

By the 1-Norm bounded smootheness of the reward function as stated in Condition 2, we have,
\begin{align} \label{eq:bounded_smoothness}
    f_{\bar \bP_t}(\bS_t) - f_{ \bP^*}(\bS_t) \leq B \sum_{e \in  \tilde \cE_t} |\bar p_{e,t} -  p_{e}^{*} ]|
\end{align}
in which $\tilde \cE_{t}$ is the set of observed edges at interaction $t$ (i.e. the set of triggered edges by $\cS_t$ at time $t$ if map it to the 1-Norm Bounded Smoothness condition). 

Substituting Eq \eqref{eq:bounded_smoothness} to Eq \eqref{eq:regret_decomposition} and combining the definition of cumulateive regret we have,
\begin{align}
   \text{R}^{\alpha \gamma}(T) & = \sum_{t=1}^T \bbE[R_i^{\alpha \gamma}] \\ \nonumber
   & \leq \frac{ B}{\alpha \gamma}  \sum_{t=1}^T \sum_{e \in \tilde \cE_{t}}   |\bar p_{e,t} -  p^*_{e}| + \bbP(\bar \xi_{t-1}) (L-K)
\end{align}
According to a proof similarly as in Lemma 2 of \cite{wen2017online},  the term $\bbP(\bar \xi_{t-1}) (L-K)$ can be bounded by $1$. Then combining the conclusion in Lemma \ref{lemma:sum_CB}, we have,
\begin{align}
   \textbf{R}^{\alpha \gamma}(T)  \leq    O\big( \frac{d B}{\alpha \gamma}  ( &  \sqrt{T |\cV^{\giving}| |\cV| D^{\textout}  \log (TD^{\textout})} \\ \nonumber 
    & + \sqrt{T |\cV^{\receiving}| |\cV| D^{\textin}  \log (TD^{\textin})}) \big)
\end{align}

\end{proof}


\begin{lemma} \label{lemma:sum_CB}
Denote  $\cV_{S_t}^{\receiving}$ as the set of observed giving nodes at time $t$ under seed node set $\cS_t$, $\cV_{S_t}^{\receiving}$ as the set of receiving nodes at time $t$ under seed node set $\cS_t$, we have with high probability,
\begin{align}
    & \sum_{t=1}^T \sum_{e \in \tilde \cE_{t}}   |\bar p_{e,t} -  p^*_{e}|  \\ \nonumber
    \leq  O( & d \sqrt{T |\cV^{\giving}| |\cV| D^{\textout}  \log (TD^{\textout})} \\ \nonumber 
    & + d \sqrt{T |\cV^{\receiving}| |\cV| D^{\textin}  \log (TD^{\textin})})
\end{align}
\end{lemma}

\begin{proof} [Proof of Lemma \ref{lemma:sum_CB}]
\begin{align}
\small
    &\sum_{t=1}^T \sum_{e \in \tilde \cE_{t}}   |\bar p_{e,t} -  p_{e}^{*}| \\ \nonumber
    & \leq \sum_{t=1}^T \sum_{e \in \tilde \cE_{t}} ( \alpha_{v_{g_e}}^{\beta} \lVert \hat \Beta_{g_e,i-1}\rVert_{\bA_{g_e,i-1}^{-1}}  + \alpha_{r_e}^{\theta} \lVert \hat \btheta_{v_{r_e},i-1}\rVert_{\bC_{v_{r_e}, i-1}^{-1}} + 2q^t ) \\
    & = \sum_{t=1}^T \sum_{u \in \cV_{\cS_t,}^{\giving}} \alpha_{u}^{\beta} \lVert \hat \Beta_{u,t}\rVert_{\bA_{u,t}^{-1}}  + \sum_{t=1}^T \sum_{v \in \cV_{\cS_t,}^{\receiving}}  \alpha_{v}^{\theta} \lVert \hat \btheta_{v,i-1}\rVert_{\bC_{v, i-1}^{-1}} + \sum_{t=1}^T |\tilde  \cE_{t}| 2q^t ) \\ \nonumber
    & \leq \alpha^{\beta} \sum_{t=1}^T \sum_{u \in \cV_{\cS_t,}^{\giving}}  \lVert \hat \Beta_{u,t}\rVert_{\bA_{u,t}^{-1}}  + \alpha^{\theta} \sum_{t=1}^T \sum_{v \in \cV_{\cS_t,}^{\receiving}} \lVert \hat \btheta_{v,i-1}\rVert_{\bC_{v, i-1}^{-1}} + \sum_{t=1}^T |\tilde \cE_{t}| 2q^t )
\normalsize
\end{align}
in which $\alpha^{\Beta} = \max_{v \in \cV} \alpha_v^{\Beta} $ and $\alpha^{\btheta}  = \max_{v \in \cV}  \alpha_{v}^{\btheta}$.

\begin{align}
   \sum_{t=1}^T  \sum_{u \in \cV_{\cS_t,}^{\giving}}  \lVert \hat \Beta_{u,t}\rVert_{\bA_{u,t}^{-1}} & \leq \sqrt{  T |\cV^{\giving}|  \sum_{t=1}^T  \sum_{u \in \cV_{\cS_t,}^{\giving}}  \lVert \hat \Beta_{u,t}\rVert^2_{\bA_{u,t}^{-1}}} \\ \nonumber
    & \leq \sqrt{T |\cV^{\giving}| |\cV|d \frac{D^{\textout} \log (1 + \frac{T D^{\textout}}{d \delta^2})}{\log(1 + 1/\delta^2)} }
\end{align}
in which $|\cV^{\giving}|$ is the maximum number of observed giving nodes at any time point, and the last inequality is based on Lemma \ref{lemma:CB_square_sum_bound}. 
and similarly we have,
\begin{align}
   \sum_{t=1}^T  \sum_{v \in \cV_{\cS_t,}^{\receiving}}  \lVert \hat \btheta_{v,t}\rVert_{\bC_{v,t}^{-1}} &  \leq \sqrt{  T |\cV^{\receiving}|  \sum_{t=1}^T  \sum_{v \in \cV_{\cS_t,}^{\receiving}}  \lVert \hat \btheta_{v,t}\rVert^2_{\bC_{v,t}^{-1}}} \\ \nonumber
    & \leq \sqrt{T |\cV^{\receiving}| |\cV|d  \frac{D^{\textin} \log (1 + \frac{T D^{\textin}}{d \delta^2})}{\log(1 + 1/\delta^2)}}
\end{align}
in which $|\cV^{\receiving}|$ is the maximum number of observed receiving nodes at any time point, and the last inequality is based on Lemma \ref{lemma:CB_square_sum_bound}. 

Then we obtain that,
\begin{align}
    &\sum_{t=1}^T \sum_{e \in \cE_{\cS_t}}   |\bar p_{e,t} -  p_{e}^{*}| \\ \nonumber
    & \leq  \alpha^{\Beta} \sqrt{T |\cV^{\giving}| |\cV|d \frac{D^{\textout} \log (1 + \frac{T D^{\textout}}{d \delta^2})}{\log(1 + 1/\delta^2)} } \\ \nonumber
   &  + \alpha^{\btheta} \sqrt{T |\cV^{\receiving}| |\cV|d  \frac{D^{\textout} \log (1 + \frac{T D^{\textin}}{d \delta^2})}{\log(1 + 1/\delta^2)}}  +  |\cE_{O}| \frac{2q(1-q^T)}{1-q} \\ \nonumber
 \leq O( & d \sqrt{T |\cV^{\giving}| |\cV| D^{\textout}  \log (TD^{\textin})} \\ \nonumber
     & + d \sqrt{T |\cV^{\receiving}| |\cV| D^{\textin}  \log (TD^{\textin})})
\end{align}
in which $\alpha^{\Beta} = \sqrt{d \log ( \frac{ d + T D^{\textout}}{d \delta^2 \lambda_1 })} + \frac{\lambda_1(1-q) + 2q}{\sqrt{\lambda_1}(1-q)}$, and $\alpha^{\btheta} = \sqrt{d \log (\frac{d+T D^{\textin}}{d \delta^2 \lambda_2}) } + \frac{\lambda_2(1-q) + 2q}{\sqrt{\lambda_2}(1-q)}$ and $|\cE_O|$ is the maximum number of observed edges at any interaction up to time $T$.

\end{proof}

\begin{lemma} \label{lemma:CB_square_sum_bound}
Denote  $\cV_{S_t}^{\giving}$ as the set of observed giving nodes at time $t$ under seed node set $\cS_t$, $\cV_{S_t}^{\receiving}$ as the set of receiving nodes at time $t$ under seed node set $\cS_t$, we have
\begin{align}
 \sum_{t=1}^T  \sum_{u \in \cV_{\cS_t,}^{\giving}}  \lVert \hat \Beta_{u,t}\rVert^2_{\bA_{u,t}^{-1}}  \leq  |\cV| \frac{d D^{\textout} \log (1 + \frac{T D^{\textout}}{d \delta^2})}{\log(1 + 1/\delta^2)} 
\end{align}
and 
\begin{align}
 \sum_{t=1}^T  \sum_{v \in \cV_{\cS_t,}^{\receiving}}  \lVert \hat \btheta_{v,t}\rVert^2_{\bC_{v,t}^{-1}} \leq |\cV|  \frac{d D^{\textin} \log (1 + \frac{T D^{\textin}}{d \delta^2})}{\log(1 + 1/\delta^2)} 
\end{align}
\end{lemma}
\begin{proof} [Proof of Lemma \ref{lemma:CB_square_sum_bound}]
The proof of this lemma is similar as Lemma 1 in \cite{wen2017online}. 
\end{proof}

\begin{proof}[Proof of Lemma \ref{lemma:activation_prob_CB}]
The proof of inequalities $\lVert \hat \btheta_{v,t} - \btheta_v^* \rVert_{\bA_{v,t}} \leq \alpha_{v}^{\Beta}$ and $\lVert \hat \Beta_{v,t} - \Beta_v^*  \rVert_{\bC_{v,t}} \leq \alpha_{v}^{\btheta}$ are technically similar as the proof in Lemma 1 in \cite{wang2016learning}, and thus omitted here. 
For the simplicity of notations, we denote the head node and tail node of edge $e$ as $u$ and $v$ respectively.
\begin{align}
   & p^*_{e} -\hat p_{e, t} \\ \nonumber
   & = {\btheta_{u}^*}^\mt \Beta_{v}^* - \hat \btheta_{u,t}^\mt \hat \Beta_{v,t} \\ \nonumber
   & = {\btheta_{u}^*}^\mt\Beta_{v}^* - {\btheta^*_{u}}^\mt \hat \Beta_{v,t} + {\btheta^*_{u}}^\mt \hat \Beta_{v,t}  - \hat \btheta_{u}^\mt \hat \Beta_{v,t} \\ \nonumber
   & = {\btheta_{u}^*}^\mt (\Beta_{v}^* -  \hat \Beta_{v,t}) + (\btheta^*_{u}   - \hat \btheta_{u,t} )^\mt\hat \Beta_{v,t} \\ \nonumber
   & = (\btheta_{u}^*-\hat\btheta_{u,t})^\mt (\Beta_{v}^* -  \hat \Beta_{v,t }) + \hat\btheta_{u}^\mt (\Beta_{v}^* -  \hat \Beta_{v,t}) + (\btheta^*_{u}   - \hat \btheta_{u,t} )^\mt \hat \Beta_{v,t}  \\ \nonumber
   & \leq  \lVert \btheta_{u}^*-\hat \btheta_{u,t} \rVert \lVert \Beta_{v}^* -  \hat \Beta_{v,t} \rVert + \lVert \hat\btheta_{u,t} \rVert \lVert \Beta_{v}^* -  \hat \Beta_{v,t}\rVert + \lVert \btheta^*_{u}   - \hat \btheta_{u,t} \rVert \lVert \hat \Beta_{v,t} \rVert \\ \nonumber
   & \leq \lVert \btheta_{u}^*-\hat\btheta_{u,0} \rVert \lVert \Beta_{v}^* -  \hat \Beta_{v,0} \rVert q^{2t} \\ \nonumber
   & +  \lVert \Beta_{v}^* -  \hat \Beta_{v,t} \rVert_{C_{u,t}} \lVert \hat\btheta_{u,t} \rVert_{C_{u,t}^{-1}} + \lVert \btheta^*_{u}   - \hat \btheta_{u,t}  \rVert_{A_{u,t}} \lVert \hat \Beta_{v,t} \rVert_{A^{-1}_{u,t}} \\ \nonumber
   & \leq 2q^{2t} + \alpha^{\theta}_{u}  \lVert \hat \btheta_{u,t} \rVert_{C_{v,t}^{-1}} + \alpha^{\beta}_{v} \lVert \hat \Beta_{v,t} \rVert_{A^{-1}_{u,t}}
\end{align}
in which $q$ is a constant in the range of $(0,1)$, and the second inequality is based on the q-linear
convergence rate of parameter estimation \cite{LocalConvergenceALS, wang2016learning}.
\end{proof}

\end{document}